\documentclass{article}

\PassOptionsToPackage{numbers}{natbib}
\usepackage[preprint]{neurips_2022}

\usepackage{amsmath,amsthm,amssymb}

\usepackage[T1]{fontenc}    % use 8-bit T1 fonts
\usepackage{hyperref}       % hyperlinks
\usepackage{url}            % simple URL typesetting
\usepackage[dvipsnames]{xcolor}
\usepackage{booktabs}       % professional-quality tables
\usepackage{amsfonts}       % blackboard math symbols
\usepackage{nicefrac}       % compact symbols for 1/2, etc.
\usepackage{microtype}      % microtypography
\usepackage[scaled]{beramono}
\usepackage{tikz-cd}
\usepackage{adjustbox}  
\usepackage{bbold}
\usepackage{longtable, makecell}
\usepackage{todonotes}
\usepackage{cleveref}
\usepackage{mathtools}
\usepackage{tabularx}
\usepackage{bbold}
\usepackage[inline]{enumitem}
\usepackage{arydshln}
\usepackage{bold-extra}
\usepackage{multirow}
\usepackage{bibunits}

\defaultbibliography{higher_order_kernels}
\defaultbibliographystyle{plain}

\newtheorem{innercustomgeneric}{\customgenericname}
\providecommand{\customgenericname}{}
\newcommand{\newcustomtheorem}[2]{%
  \newenvironment{#1}[1]
  {%
   \renewcommand\customgenericname{#2}%
   \renewcommand\theinnercustomgeneric{##1}%
   \innercustomgeneric
  }
  {\endinnercustomgeneric}
}
\newcustomtheorem{customthm}{Theorem}

\newtheorem{theorem}{Theorem}
\theoremstyle{definition}
\newtheorem{definition}{Definition}

\newtheorem{remark}{Remark}

\expandafter\def\csname ver@etex.sty\endcsname{3000/12/31}

\usepackage{autonum}

\newcommand{\Z}{\mathbb{Z}}
\newcommand{\N}{\mathbb{N}}

\newcommand{\R}{\mathbb{R}}

\newcommand{\PP}{\mathbb{P}}

\newcommand{\E}{\mathbb{E}}

\newcommand{\cN}{\mathcal{N}}

\newcommand{\cK}{\mathcal{K}}
\newcommand{\cG}{\mathcal{G}}
\newcommand{\cL}{\mathcal{L}}

\newcommand{\cV}{\mathcal{V}}
\newcommand{\cE}{\mathcal{E}}

\newcommand{\bX}{\mathbf{X}}

\newcommand{\bL}{\mathbf{L}}
\newcommand{\bc}{\mathbf{c}}
\newcommand{\bB}{\mathbf{B}}

\newcommand{\one}{\mathbb{1}}
\newcommand{\bv}{\mathbf{v}}
\newcommand{\bw}{\mathbf{w}}
\newcommand{\bu}{\mathbf{u}}
\newcommand{\bx}{\mathbf{x}}

\newcommand{\by}{\mathbf{y}}

\newcommand{\wA}{\widetilde{A}}
\newcommand{\wB}{\widetilde{B}}
\newcommand{\wC}{\widetilde{C}}
\newcommand{\wcL}{\widetilde{\mathcal{L}}}
\newcommand{\wcK}{\widetilde{\mathcal{K}}}
\newcommand{\wP}{\widetilde{P}}

\newcommand{\bell}{{\boldsymbol{\ell}}}

\newcommand{\fm}{\varphi} % Feature map on a set
\newcommand{\fms}{\widetilde{\varphi}} % Feature map for sequences
\newcommand{\fmn}{\Phi} % Feature map for a node given a random process
\newcommand{\fmg}{\Psi} % Feature map for a graph}
\newcommand{\fmnz}{\hat{\Phi}}
\newcommand{\nf}{f} % Graph features
\newcommand{\vs}{\R^d} % Vector space (state space for initial node features)

\newcommand{\tensalg}{H} % tensor algebra (not power series)
\newcommand{\tensalgps}{H} % power series tensor algebra
\newcommand{\tensalgmr}{H^{n \times n}} % matrix ring of tensors

\renewcommand{\given}{\, | \,}

\newcommand{\norm}[1]{\left\lVert #1 \right\rVert}
\newcommand{\inner}[1]{\left\langle #1 \right\rangle}

\newcommand{\inc}{\text{inc}}
\newcommand{\zs}{\text{zs}}
\newcommand{\tp}{\text{tp}}

\newcommand{\Seq}{\mathsf{Seq}}

\newcommand{\Prob}{\mathsf{Prob}}

\newcommand{\GTN}{\texttt{G2TN} }
\newcommand{\GTAN}{\texttt{G2T(A)N} }

\newcommand{\middashrule}{\hdashline\noalign{\vskip0.5ex}}

\title{Capturing Graphs with Hypo-Elliptic Diffusions}

\author{%
  Csaba Toth$^*$ \\
  Mathematical Institute\\
  University of Oxford\\
  \texttt{toth@maths.ox.ac.uk} \\
   \And
   Darrick Lee$^*$\\
   Department of Mathematics\\
   EPFL \\
   \texttt{darrick.lee@epfl.ch} \\
   \AND
   Celia Hacker \\
   Department of Mathematics\\
   EPFL \\
   \texttt{celia.hacker@epfl.ch} \\
   \And
   Harald Oberhauser \\
  Mathematical Institute\\
  University of Oxford\\
  \texttt{oberhauser@maths.ox.ac.uk} \\
}

\begin{document}

% \begin{bibunit}

\maketitle

\def\thefootnote{*}\footnotetext{Equal contribution; order determined by random coin flip.}\def\thefootnote{\arabic{footnote}}

\begin{abstract}
Convolutional layers within graph neural networks operate by aggregating information about local neighbourhood structures; one common way to encode such substructures is through random walks.
The distribution of these random walks evolves according to a diffusion equation defined using the graph Laplacian.
We extend this approach by leveraging classic mathematical results about hypo-elliptic diffusions. This results in a novel tensor-valued graph operator, which we call the hypo-elliptic graph Laplacian.
We provide theoretical guarantees and efficient low-rank approximation algorithms. In particular, this gives a structured approach to capture long-range dependencies on graphs that is robust to pooling. Besides the attractive theoretical properties, our experiments show that this method competes with graph transformers on datasets requiring long-range reasoning but scales only linearly in the number of edges as opposed to quadratically in nodes.
\end{abstract}

% \clr{\textbf{Note:} (D) I put the notation table, along with notes on bold/subscript/superscript conventions in the appendix (\Cref{apx:notation}); I think it would be nice to keep this in the final paper. Please check there for any macros.}

\everypar{\looseness=-1}

\section{Introduction}

  Obtaining a latent description of the non-Euclidean structure of a graph is central to many applications.
  One common approach is to construct a set of features for each node that represents the local neighborhood of this node; pooling these node features then provides a latent description of the whole graph.
  A classic way to arrive at such node features is by random walks: at the given node one starts a random walk, and extracts a summary of the local neighbourhood from its sample trajectories.
  We revisit this random walk construction and are inspired by two classical mathematical results:
\begin{description}
\item[Hypo-elliptic Laplacian.]
In the Euclidean case of Brownian motion $B=(B_t)_{t \ge 0}$ evolving in $\R^n$, the quantity $u(t,x)= \E[f(B_t)|B_0=x]$ solves the heat equation $\partial_t u= \Delta u$ on $[0,\infty) \times \R^n$, $u(0,x)=f(x)$.
Seminal work of Gaveau \cite{gaveau_principe_1977} in the 1970s shows that if one replaces $f(B_t)$ in the expectation by a functional of the whole trajectory, $F(B_s : s \in [0,t])$, then a path-dependent heat equation can be derived
where the classical Laplacian $\Delta$ is replaced by the so-called hypo-elliptic Laplacian.

\item[Free Algebras.]
A simple way to capture a sequence -- for us, a sequence of nodes visited by a random walk on a graph -- is to associate with each sequence element an element in an algebra\footnote{An algebra is a vector space where one can multiply elements; e.g.~the set of $n\times n$ matrices with matrix multiplication. This multiplication can be non-commutative; e.g.~$A\cdot B \neq B\cdot A$ for general matrices $A,B$.}  and multiply these algebra elements together. 
If the algebra multiplication is commutative, the sequential structure is lost but if it is non-commutative, this captures the order in the sequence.
In fact, by using the free associative algebra, this can be done faithfully and linear functionals of this algebra correspond to functionals on the space of sequences.
\end{description}
We leverage these ideas from the Euclidean case of $\R^d$ to the non-Euclidean case of graphs.
In particular, we construct features for a given node by sampling from a random walk started at this node, but instead of averaging over the end points, we average over path-dependent functions.
Put informally, instead of asking a random walker that started at a node, \emph{"What do you see now?"} after $k$ steps, we ask \emph{"What have you seen along your way?"}.
The above notions from mathematics about the hypo-elliptic Laplacian and the free algebra allow us to formalize this in the form of a generalized graph diffusion equation and we develop algorithms that make this a scalable method.

\paragraph{Related Work.}
From the ML literature, \cite{perozzi2014deepwalk, grover2016node2vec} popularized the combination of deep learning architectures to capture random walk histories. 
Such ideas have been incorporated, sometimes implicitly, into graph neural networks (GNN)
\cite{scarselli2008graph,bruna2013spectral,schlichtkrull2018modeling,defferrard2016convolutional,hamilton2017inductive,battaglia2016interaction,kipf_semi-supervised_2017} that in turn build on convolutional approaches \cite{lecun1995convolutional,lecun1998gradient,grover2016node2vec}, as well as their combination with attention or message passing \cite{monti2017geometric,velickovic_graph_2018,gilmer2017neural}, and more recent improvements \cite{xu_how_2019,morris2019weisfeiler,maron2019provably,chen2019equivalence} that provide and improve on theoretical guarantees.
Another classic approach are graph kernels, see \cite{borgwardt2020graph} for a recent survey; in particular, the seminal paper \cite{Kondor2002DiffusionKO} explored the connection between diffusion equations and random walkers in a kernel learning context. 
More recently, \cite{chen2020convolutional} proposed sequence kernels to capture the random walk history. Furthermore,~\cite{cochrane2021sk} uses the signature kernel maximum mean discrepancy (MMD) \cite{chevyrev_signature_2018} as a metric for trees which implicitly relies on the algebra of tensors that we use, and \cite{nikolentzos2020random} aggregates random walk histories to derive a kernel for graphs.
Moreover, the concept of network motifs \cite{Milo2002NetworkMS,schwarze2021motifs} relates to similar ideas that describe a graph by node sequences. 
Directly related to our approach is the topic of learning diffusion models \cite{Klicpera2019DiffusionIG,Chamberlain2021GRANDGN,thorpe_grand_2022,elhag2022graph,beltrami} on graphs. 
While similar ideas on random walks and diffusion for graph learning have been developed by different communities, our proposed method leverages these perspectives by capturing random walk histories through a novel diffusion operation.
% our hypo-elliptic graph diffusion unifies\todo{H: i think this is an overstatement} these perspectives by leveraging classical tools from mathematics.
% As the above list demonstrates, similar ideas have been developed by different communities.
% We revisit the same underlying idea of aggregating local neighbourhood information with a random walk but our main contribution is that we do this by leveraging classic tools from mathematics about the hypo-elliptic Laplacian.

Our main mathematical influence is the seminal work of Gaveau  \cite{gaveau_principe_1977} from the 1970s that shows how Brownian motion can be lifted into a Markov process evolving in a free algebra to capture path-dependence.
This leads to a heat equation governed by the hypo-elliptic Laplacian.
% ,  that again solves a heat equation.
These insights had a large influence in PDE theory, see \cite{Rothschild1976HypoellipticDO,hormander1967hypoelliptic}, but it seems that their discrete counterpart on graphs has received no attention despite the well-developed literature on random walks on graphs and general non-Euclidean objects, \cite{woess2000random,diaconis1988group, grigoryan2009heat, varopoulos1992analysis}.
A key challenge to go from theory to application is to handle the computational complexity.
To do so, we build on ideas from \cite{toth_seq2tens_2020} to design effective algorithms for the hypo-elliptic graph diffusion.

\paragraph{Contribution and Outline.}
We introduce the hypo-elliptic graph Laplacian which allows to effectively capture random walk histories through a generalized diffusion model.
\begin{itemize}
    \item In~\Cref{sec:diffusion}, we introduce the hypo-elliptic variants of standard graph matrices such as the adjacency matrix and (normalized) graph Laplacians. These hypo-elliptic variants are formulated in terms of tensor-valued matrices rather than scalar-valued matrices, and can be manipulated using linear algebra in the same manner as the classical setting.
    \item The hypo-elliptic Laplacian leads to a corresponding diffusion model, and in~\Cref{thm:non_abelian_laplacian}, we show that the solution to this generalized diffusion equation summarizes the microscopic picture of \emph{the entire history of random walks} and not just their location after $k$ steps.
    \item This solution provides a rich description of the local neighbourhood about a node, which can either be used directly as node features or be pooled over the graph to obtain a latent description of the graph. \Cref{thm:characterizing_rw_informal} shows that these node features characterize random walks on the graph, and we provide an analogous statement for graph features in~\Cref{apx:characterizing_rw}.
    \item In principle, one can solve the hypo-elliptic graph diffusion equation directly with linear algebra, but this is computationally prohibitive and Theorem \ref{thm:algo} provides an efficient low-rank approximation.
    \item Finally, Section \ref{sec:experiments} provides experiments and benchmarks. A particular focus is to test the ability of our model to capture long-range interactions between nodes and the robustness of pooling operations which makes it less susceptible to the "over-squashing" phenomenon \cite{alon2021on}.
\end{itemize}

\section{Sequence Features by Non-Commutative Multiplication.}\label{sec:sequence features}
% Given a finite-dimensional vector space $\vs$, 
We define the \emph{space of sequences in $\vs$} by
\begin{align}
    \Seq(\vs) \coloneqq \bigcup_{k=0}^\infty (\vs)^{k+1},
\end{align}
where elements are sequences denoted by $\bx = (x_0, x_1, \ldots, x_k) \in (\vs)^{k+1}$.
% \todo{k=0?}
Assume we are given an injective map, which we call the \emph{algebra lifting},
\[
\fm: \vs \to H.
\]
from $\vs$ into an algebra $H$.
We can use this to define a \emph{sequence feature map}\footnote{There are variants of this sequence feature map, which are discussed in~\Cref{apx:variations}.}
\begin{equation}
\label{eq:sequence_feature_map}
\fms: \Seq(\vs) \rightarrow H,\quad\fms(\bx) = \fm(\delta_0\bx) \cdots \fm(\delta_{k}\bx),
\end{equation}
where $\delta_0 \bx=x_0$ and $\delta_i\bx \coloneqq x_i - x_{i-1}$ for $i\ge 1$ are used to denote the \emph{increments} of a sequence $\bx=(x_0,\ldots,x_k)$. This map associates to any sequence $\bx \in \Seq(\vs)$ an element of the algebra $H$.
If the multiplication in $H$ is commutative, then the map $\fms$ would have no information about the order of increments, i.e. $\fm(\delta_0\bx) \cdots \fm(\delta_{k}\bx) = \fm(\delta_{\pi(0)}\bx) \cdots \fm(\delta_{\pi(k)}\bx)$ for any permutation $\pi$ of $\{0,\ldots,k\}$.
However, if the multiplication in $H$ is "non-commutative enough" we expect $\fms$ to be injective.

\paragraph{A Free Construction.}
Many choices for $H$ are possible, but intuitively it makes sense to use the "most general object" for $H$.
The mathematically rigorous approach is to use the \emph{free algebra over $\vs$} and we give a summary in Appendix \ref{apx:tensors}.
Despite this abstract motivation, the algebra $H$ has a concrete form: it is realized as a sequence of tensors in $\vs$ of increasing degree, and is defined by
\begin{align}
\label{eq:free_algebra}
    H \coloneqq \{ \bv = (\bv_0, \bv_1,\bv_2,\ldots): \bv_m \in (\vs)^{\otimes m}, \,m \in \N,\,  \|\bv\| < \infty\},
\end{align}
where by convention $(\vs)^{\otimes 0}=\R$, and we describe the norm $\|\bv\|$ in the paragraph below.
For example, if $\bv=(\bv_m)_{m \ge 0} \in H$, then $\bv_0$ is a scalar, $\bv_1$ is a vector, $\bv_2 \in (\vs)^{\otimes 2}$ is a $d \times d$ matrix, and so on.
The vector space structure of $H$ is given by addition and scalar multiplication according to 
\begin{align}
    \bv + \bw \coloneqq (\bv_m + \bw_m)_{m \ge 0} \in H\quad \text{ and }\quad  \lambda \bv \coloneqq (\lambda \bv_m)_{m \ge 0} \in H
\end{align}
for $\lambda \in \R$,
and the algebra structure is given by
\begin{align}
\label{eq:algebra_multiplication}
    \bv \cdot \bw \coloneqq \left( \sum_{i=0}^m \bv_i \otimes \bw_{m-i}\right)_{m \ge 0} \in H.
\end{align}
\paragraph{An Inner Product.}\looseness=-1
If $e^1,\ldots,e^d$ is a basis of $\vs$, then every tensor $\bv_m\in (\vs)^{\otimes m}$ can be written as 
\[
\bv_{m} = \sum_{1\leq i_1,\ldots,i_m \leq d} c_{i_1,\ldots,i_m} e^{i_1} \otimes \cdots \otimes e^{i_m}.
\]
This allows us to define an inner product $\langle \cdot, \cdot \rangle_m$ on $(\vs)^{\otimes m}$ by extending
\begin{equation}
\label{eq:tensor_inner_product}
\langle e^{i_1}\otimes \cdots \otimes e^{i_m}, e^{j_1}\otimes \cdots \otimes e^{j_m}\rangle_m = \left\{\begin{array}{cl}1 &: i_1=j_1,\ldots,i_m=j_m,\\0 &: \text{otherwise.}\end{array}\right.
\end{equation}
to $(\vs)^{\otimes m}$ by linearity.
This gives us an inner product on $H$,
\[
\langle \bv, \bw \rangle \coloneqq \sum_{m \ge 0} \langle \bv_m, \bw_m\rangle_m
\]
such that $H$ is a Hilbert space; in particular we get a norm $\|\bv\| \coloneqq \sqrt{ \langle \bv,\bv \rangle}$.
To sum up, the space $H$ has a rich structure: it has a vector space structure, it has an algebra structure (a noncommutative product), and it is a Hilbert space (an inner product between elements of $H$ gives a scalar).

\paragraph{Characterizing Random Walks.}
We have constructed a map $\fms$ that maps a sequence $\bx \in \Seq(\R^d)$ of arbitrary length into the space $H$. 
Our aim is to apply this to the sequence of node attributes corresponding to random walks on a graph.
Therefore, the expectation of $\fms$ should be able to characterize the distribution of the random walk.
% a desirable property of $\fms$ should be that its expectation characterizes the distribution of the underlying random walk. 
Formally the map $\fms$ is \emph{characteristic} if the map $\mu \mapsto \E_{\bx \sim \mu}[\fms(\bx)]$ from the space of probability measures on $\Seq(\R^d)$ into $H$ is injective. 
Indeed, if the chosen lifting $\fm$ satisfies some mild conditions this holds for $\fms$; see~\Cref{apx:sequence features} and \cite{chevyrev_signature_2018,toth_seq2tens_2020}. 

\paragraph{Linear Functionals.}
The quantity $\E_{\bx \sim \mu}[\fms(\bx)]$ characterizes the probability measure $\mu$ but is valued in the infinite-dimensional Hilbert space $H$. 
Using the inner product, we can instead consider
\begin{align}\label{eq: linear functionals of expected sig}
\langle \bell, \E_{\bx \sim \mu}[\fms(\bx)]\rangle  \text{ for } \bell = (\bell_0,\bell_1,\bell_2,\ldots, \bell_M,0,\ldots) \in H \text{ and }M \ge 1
\end{align}
which is equivalent to knowing $\E_{\bx \sim \mu}[\fms(\bx)]$; i.e.~the set \eqref{eq: linear functionals of expected sig} characterizes $\mu$. 
This is analogous to how one can use either the moment generating function of a real-valued random variable or its sequence of moments to characterize its distribution; the former is one infinite-dimensional object (a function), the latter is a infinite sequence of scalars.
We extend a key insight from \cite{toth_seq2tens_2020} in~\Cref{sec:algos}: a linear functional $\langle \bell, \E_{\bx \sim \mu}[\fms(\bx)] \rangle$ can be efficiently approximated without directly computing $\E_{\bx \sim \mu}[\fms(\bx)]$ or storing large tensors.

\paragraph{The Tensor Exponential.}
While we will continue to keep $\fm$ arbitrary for our main results (see~\cite{toth_seq2tens_2020} and~\Cref{apx:variations} for other choices), we will use the \emph{tensor exponential} $\exp_\otimes: \vs \rightarrow H$, defined by 
\begin{equation}
\label{eq:tensor_exponential}
    \exp_{\otimes}(x) = \left(\frac{x^{\otimes m}}{m!}\right)_{m\geq 0},
\end{equation}
as the primary example throughout this paper and in the experiments in~\Cref{sec:experiments}. With this choice, the induced sequence feature map is the discretized version of a classical object in analysis, called the path signature, see \Cref{apx:sequence features}.

%\todo{is there some abstract statement about the existence of such an embedding? if $H$ is Lie group, then compose free lie algebra functor, with universal enveloppping algebra to get this embedding. D: I think it would be best to say that $\tensalg$ is the free associative algebra over $V$. Unless we are specifically dealing with Lie groups, it's unnecessary to use the universal enveloping algebra.} 

\section{Hypo-Elliptic Diffusions}\label{sec:diffusion}
Throughout this section, we fix a labelled graph $\cG = (\cV, \cE, \nf)$, that is $\cV$ is a set of $n$ nodes $\cV = \{1, \ldots, n\}$, $\cE$ denotes edges and $\nf:\cV \to \vs$ is the set of continuous node attributes\footnote{The labels given by the labelled graph are called \emph{attributes}, while the computed updates are called \emph{features}.} which map each node to an element in the vector space $\vs$.
Two nodes $i, j \in \cV$ are \emph{adjacent} if $(i,j) \in \cE$ is an edge, and we denote this by $i \sim j$. 
The \emph{adjacency matrix} $A$ of a graph is defined by
\[
A_{i,j} =
\left\{
\begin{array}{cl}
  1 & :i \sim j\\
  0 &: \text{otherwise}.
\end{array}\right.
\]
We use $\operatorname{deg}(i)$ to denote the number nodes that are adjacent to node $i$.

\paragraph{Random Walks on Graphs.}
Let $(B_k)_{k\ge 0}$ be the simple random walk on the nodes $\cV$ of $\cG$, where the initial node is chosen uniformly at random. The \emph{transition matrix} of this time-homogeneous Markov chain is 
\[
P_{i,j} \coloneqq \PP(B_k = j| B_{k-1}=i) = \left\{ \begin{array}{cl}\frac{1}{\operatorname{deg}(i)} &: i\sim j\\ 0 &: \text{otherwise}.\end{array}\right.
\]
Denote by $(L_k)_{k \ge 0}$ the random walk lifted to the node attributes in $\vs$, that is
\begin{equation}
\label{eq:lifted_random_walk}
    L_k \coloneqq \nf(B_k).
\end{equation}
% In the case\footnote{Graph diffusion and the graph Laplacian can also be defined for vector-valued attributes using block matrices, but we restrict our exposition of classical diffusion to $d=1$ for simplicity.} of $d=1$, 
Recall that the \emph{normalized graph Laplacian} for random walks is defined as $\cL =I-D^{-1}A$, where $D$ is diagonal degree matrix; in particular, the entry-wise definition is 
 \begin{align}
    \label{eq:laplacian_matrix}
        \cL_{i,j} \coloneqq \left\{
        \begin{array}{cl}
            -\frac{1}{\operatorname{deg}(i)} &  : i \sim j \\
            1 & : i=j\\
            0 & : \text{otherwise}.
        \end{array}
        \right.
    \end{align}

The discrete graph diffusion equation for $U_k \in \R^{n \times d}$ is given by
\begin{align}\label{eq:heat equation}
    U_k - U_{k-1} = -\cL U_{k-1}, \quad U^{(i)}_0 = \nf(i)
\end{align}
where the initial condition $U_0 \in \R^{n \times d}$ is specified by the node attributes.\footnote{The attributes over all nodes are given by an $n \times d$ matrix; in particular $U_k^{(i)}$ is the $i^{\text{th}}$ row of the matrix.}  The probabilistic interpretation of the solution to this diffusion equation is classical and given as
\begin{equation}\label{eq:classic graph diffusion}
    U_k = \left(\E[L_k \given B_0 = i]\right)_{i=1}^n = P^k U_0.
\end{equation}
This allows us to compute the solution $u_k$ using the transition matrix $P = I - \cL$.
% which can be obtained from the graph Laplacian as $P = I - \cL$. 

\paragraph{Random Walks on Algebras.}
% Returning to the general setting of vector-valued attributes $f: \cV \rightarrow \vs$, 
We now incorporate the history of a random walker by considering the quantity
% \todo{Cs: Make $\delta$ part of $\fms$ + sentence about parametrization?}
\begin{align}\label{eq:narw}
\E[\fms(\bL_k) \given B_0 = i]=\E[\fm(\delta_0 \bL) \cdots \fm(\delta_k \bL)  \given B_0 = i]
\end{align}
where $\bL_k = (L_0, \ldots, L_k)$.
% Note that in the expectation we compute the sequence feature map of the random walk $\bL_k = (0, L_0, \ldots, L_k)$, where the pre-appended $0$ allows us to take into account the starting point $L_0 = \nf(i)$. 
Since $\fms$ captures the whole history of the random walk $\bL_k$ over node attributes, we expect this expectation to provide a much richer summary of the neighborhood of node $i$ than $\E[L_k|B_0=i]$. 
The price is however, the computational complexity, since \eqref{eq:narw} is $H$-valued.
We first show, that analogous to \eqref{eq:heat equation}, the quantity \eqref{eq:narw} satisfies a diffusion equation that can be computed with linear algebra. 
To do so, we develop a graph analogue of the hypo-elliptic Laplacian and replace the scalar entries of the matrices with entries from the algebra $H$.

\paragraph{Matrix Rings over Algebras.}

We first revisit the adjacency matrix $A \in \R^{n \times n}$ and replace it by the \emph{tensor adjacency matrix} $\wA= (\wA)_{i,j} \in \tensalgmr$, that is $\wA$ is a matrix but instead of scalar entries its entries are elements in the algebra $H$.
The matrix $A$ has an entry at $i,j$ if nodes $i$ and $j$ are connected; $\wA$ replaces the $i,j$ entry with an element of $H$ that tells us how the node attributes of $i$ and $j$ differ,
 \begin{align}
    \label{eq:exponential_adjacency_matrix}
        \wA_{i,j} \coloneqq \left\{
        \begin{array}{cl}
            \fm(\nf(j) - \nf(i)) &  : i \sim j \\
            0 & : \text{otherwise}.
        \end{array}
        \right.
    \end{align}
Matrix multiplication works for elements of $\tensalgmr$ by replacing scalar multiplication by multiplication in $H$, that is $(\wB\cdot\wC)_{i,j} = \sum_{k=1}^n \wB_{i,k} \cdot \wC_{k,j}$ for $\wB, \wC \in \tensalgmr$ and $\wB_{i,k} \cdot \wC_{k,j}$ denotes multiplication in $H$ as in Equation~\eqref{eq:algebra_multiplication}.
For the classical adjacency matrix $A$, the $k$-th power counts the number of length $k$ walks in the graph, so that $(A^k)_{i,j} $ is the number of walks of length $k$ from node $i$ to node $j$.
%We can take powers of $\wA$ in the same way as in the classical case e.g.~$\wA^2_{i,j} = \sum_{k}\wA_{i,k} \otimes \wA_{k,j}$, but using the multiplication in $T_d$ for the element-wise multiplication in this sum. 
%The formal statement is that $\wA$ is an element of the matrix ring $\operatorname{Mat}_{|\vs|}(T_d) = T_d^{|\vs| \times |\vs|}$ where we identify $T_d$ as a ring.
We can take powers of $\wA$ in the same way as in the classical case, where
% the interpretation of the $(i,j)$-entry is
\begin{equation}\label{eq:powers_tensor_adjacency}
    (\wA^k)_{i,j} = \sum_{\bx} \fm(\delta_1\bx) \cdots \fm(\delta_k\bx)
\end{equation}
where the sum is taken over all length $k$ walks $\bx = (f(i), \ldots f(j))$ from node $i$ to node $j$ (full details are provided in~\Cref{apx:details_diffusion}).
Since $\fms(\bx)$ characterizes each walk $\bx$, the entry $\wA^k_{i,j}$ can be interpreted as a summary of all walks which connect nodes $i$ and $j$.
% \todo{show this in appendix}

\paragraph{Hypo-elliptic Graph Diffusion.}
% \begin{itemize}
%     \item Keep everything here, but use general diffusion in the formal theorem statement
%     \item Start the formulation off with continuous diffusion on a graph (since we want to talk about that later)
%     \item Also write down the discrete-time equations as they are written now
% \end{itemize}
% So far, we have focused on the microscopic point of view of a stochastic process evolving on the graph. 
% But it turns out, that analogous to the classic random walk, one can formulate a macroscopic heat equation. 
% The key is that the classical (normalized) graph Laplacian can be replaced by an appropriate Laplacian acting on tensors.

Similar to the tensor adjacency matrix, we define the \emph{hypo-elliptic graph Laplacian} as the $n \times n$ matrix
\[
    \wcL = I-D^{-1}\wA \in \tensalgmr,
\]  
where $D$ is the degree matrix embedded into $\tensalgmr$ at tensor degree $0$. The entry-wise definition is
\begin{align}
    \label{eq:tensor_laplacian_matrix}
        \wcL_{i,j} \coloneqq \left\{
        \begin{array}{cl}
            \frac{-\fm(\nf(j) - \nf(i))}{\operatorname{deg}(i)} &  : i \sim j \\
            1 &: i = j \\
            0 & : \text{otherwise}.
        \end{array}
        \right.
\end{align}
We can now formulate the \emph{hypo-elliptic graph diffusion equation} for $\bv_k \in \tensalgps^n$ as
\begin{equation}
\label{eq:non_abelian_diffusion_equation}
    \bv_k - \bv_{k-1} = - \wcL \bv_{k-1}, \quad \bv^{(i)}_0 = \fm(\nf(i)).
\end{equation}
Analogous to the classic graph diffusion \eqref{eq:classic graph diffusion}, the hypo-elliptic graph diffusion \eqref{eq:non_abelian_diffusion_equation} has a probabilistic interpretation in terms of $\bL$ as shown in Theorem \ref{thm:non_abelian_laplacian} (the proof is given in~\Cref{apx:details_diffusion}).
\begin{theorem}
\label{thm:non_abelian_laplacian}
    Let $k \in \N$, $\bL_k = (L_0, \ldots, L_k)$ be the lifted random walk from~\eqref{eq:lifted_random_walk}, and $\wP = I - \wcL$ be the \emph{tensor adjacency matrix}. The solution to the hypo-elliptic graph diffusion equation~\eqref{eq:non_abelian_diffusion_equation} is
    \[
        \bv_k = \left( \E[\fm(\delta_1\bL_k) \cdots \fm(\delta_k \bL_k) | B_0 = i]\right)_{i=1}^n = \wP^k \one_H.
    \]
    Furthermore, if $F \in H^{n \times n}$ is the diagonal matrix with $F_{i,i} = \fm(f(i))$, then
    \[
        F\bv_k = \left(\E[\fms(\bL_k)| B_0 = i]\right)_{i=1}^n.
    \]
\end{theorem}

% \begin{theorem}
% \label{thm:non_abelian_laplacian}
% Let $k \in \N$, $\bL_k = (L_0, \ldots, L_k)$ be the lifted random walk from~\eqref{eq:lifted_random_walk}, and $\wP = I - \wcL$ be the \emph{tensor transition matrix}. The solution to the hypo-elliptic graph diffusion equation \eqref{eq:non_abelian_diffusion_equation} is
% \begin{equation}
% \label{eq:non_abelian_solution}
%     \bv_k = \left(\E[\fms(\bL_k) \given B_0 = i]\right)_{i=1}^n = \wP^k \bv_0.
% \end{equation}
% % Starting with $x(0)\in \R^n$ as an initial distribution on the nodes and if $x(k)\in \tensalgps$ is the solution of 
% %         \[x(k) -x(k-1)= -x(k-1) \widetilde{\mathcal{L}},\] 
% % then for any node $v\in \cV$ \[x_v(k) = \E[\fms(\delta_1 L,\ldots,\delta_k L)|B_n =v ]\]
% % is the expected feature value of random walks of length $k$ ending at node $v$. 
% \end{theorem}
In the classical diffusion equation, $U_k$ captures the concentration of the random walkers after $k$ time steps over the nodes. In the hypo-elliptic diffusion equation, $\bv_k$ captures summaries of random walk histories after $k$ time steps over the nodes since $\fms(\bL_k)$ summarizes the whole trajectory $\bL_k=(L_0,\ldots,L_k)$ and not only the endpoint $L_k$. 
%\Cref{thm:non_abelian_laplacian} is an example of the classic insight, that although microscopic movements are random, the macroscopic picture can have a concise deterministic description. 
\paragraph{Node Features and Graph Features.}
Theorem \ref{thm:non_abelian_laplacian} can then be used to compute features $\fmn(i) \in \tensalg$ for individual nodes as well as a feature $\fmg(\cG)$ for the entire graph. The former is given by $i$-th component $\bv_k^{(i)}$ of the solution $\bv_k=(\bv_k^{(i)})_{i=1,\ldots,n} \in H^n$ of Equation~\eqref{eq:non_abelian_diffusion_equation},
\begin{equation}
\label{eq:nonabelian_node_features}
    \fmn(i) \coloneqq  \bv_k^{(i)} = \E[\fms( \bL_k) \given B_0 = i]= (F\wP^k \bv_0)^{(i)} \in H,
\end{equation}
since the random walk $B$ chooses the starting node $B_0=i$ uniformly at random. The latter can be computed by mean pooling the node features, which also has a probabilistic interpretation as
\begin{align} \label{eq:mean_pooled_features}
     \fmg(\cG) \coloneqq \frac{1}{n} \sum_{i=1}^n \bv_k^{(i)} = \E[\fms(\bL_k)] =  n^{-1} (\one_H^T F\wP^k \bv_0) \in H,
\end{align}
where $\one_H^T \coloneqq (1_H, \, \ldots, \, 1_H) \in H^n$ is the all-ones vector in $H$ and $1_H$ denotes the unit in $H$.
%since the random walk $B$ chooses the starting node $B_0=i$ uniformly at random.

\paragraph{Characterizing Graphs with Random Walks.} 
The graph and node features obtained through the hypo-elliptic diffusion equation are highly descriptive: they characterize the entire history of the random walk process if one also includes the time parametrization, as described in~\Cref{apx:sequence features}.
% The formal statement and proof of \Cref{thm:characterizing_rw_informal} is given in~\Cref{apx:characterizing_rw}.

\begin{theorem} 
\label{thm:characterizing_rw_informal}
    Suppose $\Psi$ is the graph feature map from Equation~\eqref{eq:mean_pooled_features} induced by the tensor exponential algebra lifting including time parametrization. Let $\cG$ and $\cG'$ be two labelled graphs, and $\bL_k = (L_0, \ldots, L_k)$ and $\bL'_k = (L'_0, \ldots, L'_k)$ be the $k$-step lifted random walk as defined in Equation~\eqref{eq:lifted_random_walk}. Then, $\Psi(\cG) = \Psi(\cG')$ if and only if the distributions of $\bL_k$ and $\bL'_k$ are equal. 
\end{theorem}
An analogous result holds for the node features, and we prove both results in~\Cref{apx:characterizing_rw}.
% \Cref{thm:characterizing_rw_informal} as well as the version for mean-pooled features follow as special cases from of general result that we prove in \Cref{apx:characterizing_rw}.
While we use the tensor exponential in this article, many other choices of $\fms$ are possible and result in graph and node features with such properties: under mild conditions, if the algebra lifting $\fm:\vs \to H$ characterizes measures on $\vs$, the resulting node feature map $\fmn$ characterizes the random walk, see~\cite{toth_seq2tens_2020}, which in turn implies the above results. Possible variations are discussed in~\Cref{apx:variations}.
% \todo{D: I was thinking we could just reference seq2tens for other options for $\fm$. There's not much to say other than what you've said there, and we only focus on the tensor exp anyways.}.
% Although the tensor exponential has great theoretical guarantees and empirical results, the choice of $\fms$ should be regarded as a hyper-parameter that can be optimized over which is what we do in our experiments. 

% We discuss these extensions in detail in~\Cref{apx:characterizing_rw}

% It is well-known that the sequence of moments encodes the distribution of a RV on $\R^d$, hence the exponential tensor embedding $\fm(\bx) = \exp(\bx)$ is the most straightforward choice in this sense, but other options are also possible, see \cite[App.~B.1]{toth_seq2tens_2020}. 

\paragraph{General (Hypo-elliptic) Diffusions and Attention.}
One can consider more general diffusion operators, such as the normalized Laplacian $\cK$ of a weighted graph.
% $\cK$ than $\cL$; for example, $\cK$ could be the normalized Graph Laplacian of a weighted graph. 
We define its lifted operator $\wcK \in \tensalg^{n\times n} $ analogous to Equation \eqref{eq:tensor_laplacian_matrix}, resulting in a generalization of~\Cref{thm:non_abelian_laplacian} with $\wcK$ replacing $\wcL$. 
% Then Theorem \ref{thm:non_abelian_laplacian} generalizes to the lifted operator $\wcK$, replacing the role of the hypo-elliptic Laplacian $\wcL$.
In the flavour of convolutional GNNs \cite{bronstein_geometric_2021}, we consider a weighted adjacency matrix $A \in \R^{n \times n}$
\begin{align}
    A_{i,j} = \left\{\begin{array}{cl}
        c_{i,j} & : i \sim j \\
        0 & :\text{otherwise},
    \end{array}\right.
\end{align}
for $c_{i,j} > 0$. 
The corresponding normalized Laplacian $\cK$ is given by $\cK = I - D^{-1}A$, where $D$ is a diagonal matrix with $D_{i,i} = \sum_{j \in \cN(i)} c_{i,j}$. 
A common way to learn the coefficients is by introducing parameter sharing across graphs by modelling them as $c_{i,j} = \exp(a(\nf(i), \nf(j)))$ using a local attention mechanism, $a: \R^d \times \R^d \rightarrow \R$ \cite{velickovic_graph_2018}. 
In our implementation, we use additive attention~\cite{bahdanau2015neural} given by $a(\nf(i), \nf(j)) = \texttt{LeakyRelu}_{0.2}(W_s \nf(i) + W_t \nf(j))$, where $W_s, W_t \in \R^{1 \times d}$ are linear transformations for the source and target nodes, but different attention mechanisms can also be used; e.g.~scaled dot-product attention \cite{vaswani2017attention}. Then, the corresponding transition matrix $P = D^{-1} A$ is  defined as $P_{ij} = \mathrm{softmax}_{k \in \cN(i)}(a(f(i), f(k)))_j$. The lifted transition matrix is defined as
\begin{align}
  \wP = \left\{\begin{array}{cl}
    P_{i,j} \fm(\nf(j) - \nf(i)) & : i \sim j \\
    0 & :\text{otherwise}.
  \end{array}\right.
\end{align}
The statements of Theorem \ref{thm:non_abelian_laplacian} immediately generalize to this variation by replacing the expectation with respect to a non-uniform random walk. Hence, in this case the use of attention can be interpreted as learning the transition probabilities of a random walk on the graph.
% \todo{Initialization... Cs: Probably defer it to Appendix \ref{apx:model_detail} }

% allows us to answer questions about this neighbourhood by linear functionals of $\fmn(i)$. 
% We are not limited to just one functional, and applied to with a set of $d$ functionals, $\ell_1,\ldots,\ell_d$ the above discussion gives us a way to update labels
% \[
% \R^d \to \R^d, \quad v \mapsto \begin{pmatrix} \langle \ell_1,\E[\fmn(B_0,B_1,\ldots,B_n) | B_0=v] \rangle\\ \vdots \\ \langle \ell_d, E[\fmn(B_0,B_1,\ldots,B_n) | B_0=v] \rangle \end{pmatrix}
% \]

% Because $\tensalgps$ is a vector space, we can treat $\Phi$ as the initial node features for the graph, take the free algebra of $\tensalgps$, and iterate this entire procedure in order to obtain a more descriptive summary of neighbourhoods about each node. This procedure would summarize information about \emph{walks of walks} (iterated walks) as \emph{tensors of tensors} (iterated tensors). 

%Our goal is to encode the neighbourhood about a node $i$ using $\fmn(i)$ and apply nonlinear\todo{confusing since in section 2 we emphasize linear functionals. D: Because our feature map is not universal, should we remove that part of section 2?} functions to it to answer questions about the graph. A prominent way of representing nonlinear functions is through compositions of linear maps and pointwise activations, as done in e.g.~in neural networks.
\section{Efficient Algorithms for Deep Learning}\label{sec:algos}
% We want to answer questions about a node $i$ and its neighborhood by applying linear functionals $\langle \bell, \fmn(i)\rangle \in \R$ to the node features $\fmn(i) \in H$.
 We now focus on the computation of linear maps of $\Phi(i)$ defined as $Q: H \rightarrow \R^{d^\prime}$, $\fmn(i) \mapsto (\inner{\bell^j, \fmn(i)})_{j=1, \dots, d^\prime}$ for $\bell^j = (\bell^j_0, \bell^j_1, \dots, \bell^j_M, 0, 0, \dots) \in H$. These maps capture the distribution of the random walk $\bL_k$ conditioned to start at node $i$ and pooling them captures the graph as seen by the random walk, see Section \ref{sec:diffusion}.
These linear maps are parametrized by $\bell^j$ which we view as learnable parameters of our model and hence,  requires efficient computation. For any $\bell$, we could compute $\langle \bell, \fmn(i) \rangle$ by first evaluating $\fmn(i)$ using linear algebra by \Cref{thm:non_abelian_laplacian} and subsequently taking their inner product with $\bell$. However, this involves computations with high-degree tensors, which makes this approach not scalable.
% However, although \Cref{thm:non_abelian_laplacian} reduces this to a problem of linear algebra, it involves computations with high-degree tensors, which makes this approach infeasible in practice.
%Additionally, once a linear projection $Q \circ \Phi(i)$ has been computed for all nodes $i \in \cV$, it can also be used to compute the linear projection of $Q \circ \Psi(\cG)$, since pooling is linear and commutes with $Q$.
In this section, we leverage ideas about tensor decompositions to design algorithms to compute $\langle \bell, \fmn(i)\rangle$ efficiently \cite{toth_seq2tens_2020}; in particular, we do not need to compute $\fmn(i) \in H$ directly or store large tensors. 

\paragraph{Computing a Rank-$1$ Functional.}
First, we focus on a \emph{rank-$1$} linear functional $\bell \in H$ given as
% \todo{D: I changed this so that the proof is easier to write.}
\begin{align}\label{eq:low rank functional}
\bell = (\bell_m)_{m \ge 0} \text{ with } \bell_m =u_{M-m+1} \otimes \cdots \otimes u_M \text{ and } \bell_m =0 \text{ for } m > M, 
\end{align}
where $u_m \in \vs$ for $m = 1, \ldots, M$ for a fixed $M \ge 1$. 
\Cref{thm:algo} shows that for such $\bell$, the computation of $\langle \bell, \fmnz(i)\rangle$ can be done \begin{enumerate*}[label=(\alph*)] \item efficiently by factoring this low-rank structure into the recursive computation, and \item simultaneously for all nodes $i \in \cV$ in parallel \end{enumerate*}. This can then be used to compute rank-$R$ functionals for $R>1$, and for $\langle \bell, \fmn(i)\rangle$; see~\Cref{apx:low_rank_functionals}.

\begin{theorem}\label{thm:algo}
%   and denote
%   \[
% \ell_{q,r}\coloneqq w_q \otimes \cdots \otimes w_r.
%   \]
Let $\bell$ be as in \eqref{eq:low rank functional} and define $f_{k,m} \in \R^{n}$ for $m=1, \dots, M$ as
\begin{align}
f_{1,m} \coloneqq\frac{1}{m!}\left(P \odot C^{u_{M-m+1}} \odot \cdots \odot C^{u_M}\right) \cdot \one,
\end{align}
where $\one^T \coloneqq (1, \ldots, 1) \in \R^n$ is the all-ones vector; and for $2 \leq k$ and  $1 \leq m \leq M$ recursively as 
% \todo{D: This used to say $A \cdot f_{l-1, m}$, but I think this was a typo.}
  \begin{align} \label{eq:recursive_low_rank}
   f_{k,m} \coloneqq P \cdot f_{k-1,m} + \sum_{r=1}^m \frac{1}{r!} \left(P \odot C^{u_{M-m+1}} \odot \cdots \odot C^{u_{M-m+r}}\right) \cdot f_{k-1,m-r},
  \end{align}
  where the matrix $C^u=(C^u_{i,j}) \in \R^{n \times n}$ is defined as
\begin{align}
C^u_{i,j}\coloneqq  \left\{\begin{array}{cl}
      \inner{u, \nf(j) - \nf(i)} &: i \sim j,\\
      0  &: \text{otherwise}.
    \end{array}\right.
  \end{align} 
  and $\odot$ denotes element-wise\footnote{For example
    $\begin{bmatrix}
      1& 2 \\ 
      3&4
    \end{bmatrix}
    \odot
    \begin{bmatrix}
      5&6\\7&8 
    \end{bmatrix} =
    \begin{bmatrix}
      5& 12\\
      21&32    \end{bmatrix}.
    $} multiplication, while $\cdot$ denotes matrix multiplication.
  Then, it holds for $i \in \cV$, random walk length $k \in \Z_+$, and tensor degree $m = 1, \ldots, M$, that 
  \begin{align}
    f_{k, m}(i) = \inner{\bell_m, \fmnz_k(i)},
  \end{align}
  where $\fmnz_k(i) = \E[\fm(\delta_1\bL_k) \cdots \fm(\delta_k\bL_k) \given B_0 = i]$.
\end{theorem}

\begin{remark}
  Overall, Eq.~\eqref{eq:recursive_low_rank} computes $f_{k, m}(i)$ for all $i \in \cV$, $k = 1, \ldots, K$, $m = 1, \ldots, M$ in $O(K \cdot M^2 \cdot N_E + M \cdot N_E \cdot d)$ operations, where $N_E \in \N$ denotes the number of edges; see App.~\ref{apx:low_rank_functionals}.
\end{remark}

%  The drawback of using rank-1 functionals is a loss of expressiveness: not every function of the neighbourhood about a node can be expressed as a rank-1 functional.
%  Luckily, iterations of rank-1 functionals allow to approximate any general linear functional as we discuss in the next step.
% \paragraph{Iterating Rank-1 Functionals.}
% For simplicity assume the original node features are in $\R^d$, that is $\vs=\R^d$.
% Given $d$ rank-1 functionals, $\bell^1,\ldots,\bell^d \in H$, we can apply \Cref{thm:algo} $d$-times to compute
% \begin{align}\label{eq:updated label}
%     \langle \left(\bell^j, \fmn(i) \rangle \right)_{j \in \{1,\ldots,d\}} \in \R^d
% \end{align}
% and consider the vector \eqref{eq:updated label} as a new label for node $i$. 
% Hence, each choice of $d$ rank-1 functionals allows to update the label for each node $i$. 
% Moreover, this iteration also has a probabilistic interpretation, where instead of considering one long random walk we use a random walk over random walks.
% Formalizing all this requires more notation and we give the details in~\Cref{apx:iteration} which also discusses general low-rank approximations.
% Informally, the situation is similar to the trade-off of width versus depth in neural networks: in this analogy, the rank of the functional corresponds to the width of the neural network and the number of iterations of \eqref{eq:updated label} to the depth.

\paragraph{Building Neural Networks.} Given an algorithm to compute rank-1 functionals of $\fmn(i)$, we proceed in the following steps to build neural network models: \begin{enumerate*}[label=(\arabic*)] \item \label{step1} given a random walk length $k \in \N$, max tensor degree $M \in \N$ and max tensor rank $R \in \N$, we introduce $R$ rank-1 degree-$M$ functionals $\tilde\bell_M^{j}$, and compute $f^j_{k,m}(i)$ for all $j=1,\dots,R$ and $m=1,\dots,M$ using Theorem \ref{thm:algo} that gives a linear mapping between $H \rightarrow \R^{R M}$. \item \label{step2} Apply a linear mixing layer $\R^{RM} \rightarrow \R^{R}$ to the rank-1 functionals to recover at most rank-$R$ functionals. Update the node labels, which now incorporate neighbourhood information up to hop-$k$.  \item Iterate this construction by repeating steps \ref{step1} and \ref{step2}. \end{enumerate*} 

% \todo{H: we could delete this paragraph "The drawback of using ..." if we need space. D: I agree that this could be deleted}
The drawback of using at most rank-$R$ functionals is a loss of expressiveness: not all linear functionals of tensors can be represented as rank-$R$ functionals if $R$ is not large enough. However, we hypothesize that iterations of low-rank functionals allows to approximate any general linear functional as in \cite{toth_seq2tens_2020}. Further, this has a probabilistic interpretation, where instead of considering one long random walk we use a random walk over random walks, increasing the effective random walk length additively.
% Formalizing all this requires more notation and we give the details in~\Cref{apx:iteration} which also discusses general low-rank approximations.

\section{Experiments} \label{sec:experiments}
%We have carried out experiments to evaluate our hypo-elliptic diffuson approach. 
We implemented the above approach and call the resulting model \textbf{G}raph\textbf{2T}ens \textbf{N}etworks since it represents the neighbourhood of a node as a sequence of tensors, which is further pushed through a low-tensor-rank constrained linear mapping, similarly to how neural networks linearly transform their inputs pre-activation. A conceptual difference is that in our case the non-linearity is applied first and the projection secondly, albeit the computation is coupled between these steps.

\paragraph{Experimental Setup.} The aim of our main experiment is to test the following key properties of our model: \begin{enumerate*}[label=(\arabic*)] \item ability to capture long-range interactions between nodes in a graph, \item robustness to pooling operations, hence making it less susceptible to the ``over-squashing'' phenomenon \cite{alon2021on} \end{enumerate*}. 
We do this by following the experiments in \cite{wu2021representing}. In particular, we show that our model is competitive with previous approaches for retaining long-range context in graph-level learning tasks but without computing all pairwise interactions between nodes, thus keeping the influence distribution localized \cite{xu2018representation}. We further give a detailed ablation study to show the robustness of our model to various architectural choices, and compare different variations on the algorithm from App.~\ref{apx:variations}.
As a second experiment, we follow the previous applications of diffusion approaches to graphs that have mostly considered inductive learning tasks, e.g.~on the citation datasets \cite{Chamberlain2021GRANDGN, thorpe_grand_2022, beltrami}. 
Our experimentation on these datasets are available in Appendix \ref{apx:further_exp}, where the model performs on par with short-range GNN models, but does not seem to benefit from added long-range information a-priori. However, when labels are dropped in a $k$-hop sanitized way as in \cite{rampavsek2021hierarchical}, the performance decrease is less pronounced.

\paragraph{Datasets.} We use two biological graph classification datasets (NCI1 and NCI109), that contain around ${\sim} 4000$ biochemical compounds represented as graphs with ${\sim} 30$ nodes on average~\cite{wale_comparison_2008, PubChem}. The task is to predict whether a compound contains anti-lung-cancer activity. The dataset is split in a ratio of $80\%-10\%-10\%$ for training, validation and testing. Previous work \cite{alon2021on} has found that GNNs that only summarize local structural information can be greatly outperformed by models that are able to account for global contextual relationships through the use of \textit{fully-adjacent} layers. This was further improved on by \cite{wu2021representing}, where a local neighbourhood encoder consisting of a GNN stack was upgraded with a Transformer submodule \cite{vaswani2017attention} for learning global interactions.

\paragraph{Model Details.} We build a GNN architecture primarily motivated by the GraphTrans (small) model from \cite{wu2021representing}, and only fine-tune the pre- and postprocessing layers(s), random walk length, functional degree and optimization settings. In detail, a preprocessing MLP layer with $128$ hidden units is followed by a stack of $4$ \GTN layers each with RW length-$5$, max rank-$128$, max tensor degree-$2$, all equipped with JK-connections \cite{xu2018representation} and a max aggregator. Afterwards, the node features are combined into a graph-level representation using gated attention pooling \cite{li2016gated}. The pooled features are transformed using a final MLP layer with $256$ hidden units, and then fed into a softmax classification layer. The pre- and postprocessing MLP layers employ skip-connections \cite{he2016deep}. Both MLP and \GTN layers are followed by layer normalization \cite{ba2016layer}, where GTN layers normalize their rank-$1$ functionals independently across different tensor degrees, which corresponds to a particular realization of group normalization \cite{wu2018group}. We randomly drop $10\%$ of the features for all hidden layers during training \cite{srivastava2014dropout}. The attentional variant, \GTAN also randomly drops $10\%$ of its edges and uses $8$ attention heads \cite{velickovic_graph_2018}. Training is performed by minimizing the categorical cross-entropy loss with an $\ell_2$ regularization penalty of $10^{-4}$. For optimization, Adam \cite{kingma2015adam} is used with a batch size of $128$ and an inital learning rate of $10^{-3}$ that is decayed via a cosine annealing schedule \cite{loshchilov2017sgdr} over $200$ epochs. Further intuition about the model and architectural choices are available in Appendix \ref{apx:model_detail}.

\paragraph{Baselines.} We compare against \begin{enumerate*}[label=(\arabic*)] \item the baseline models reported in \cite{wu2021representing}, \item variations of GraphTrans, \item other recently proposed hierarchical approaches for long-range graph tasks \cite{rampavsek2021hierarchical}. \end{enumerate*} Groups of models in Table \ref{table:nci_benchmark} are separated by dashed lines if they were reported in separate papers, and the first citation after the name is where the result first appeared. The number of GNN layers in HGNet are not discussed by \cite{rampavsek2021hierarchical}, and we report it as implied by their code. We organize the models into three groups divided by solid lines: \begin{enumerate*}[label=(\alph*)] \item \label{baseline:local} baselines that only apply neighbourhood aggregations, and hierarchical or global pooling schemes,  \item \label{baseline:pairwise} baselines that first employ a local neighbourhood encoder, and afterwards fully densify the graph in one way or another so that all nodes interact with each other \emph{directly}, \item our models that we emphasize thematically belong to \ref{baseline:local} \end{enumerate*}.

\begin{table}[t]
  \vspace{-5pt}
  \caption{Comparison of classification accuracies on NCI biological datasets, where we report mean and standard deviation over $10$ random seeds for our models.}
  \label{table:nci_benchmark}

  \centering
  \begin{small}
  \begin{tabular}{lcccc}
    \toprule
    \textbf{Model} & \textbf{GNN Type} & \textbf{GNN Count} & \textbf{NCI1 (\%)} & \textbf{NCI109 (\%)} \\
    \midrule
    Set2Set \cite{lee2019self, vinyals2016order} & GCN  & $3$ & $68.6 \pm 1.9$ & $69.8 \pm 1.2$ \\
    SortPool \cite{lee2019self,zhang2018end} & GCN & $3$ & $73.8 \pm 1.0$ & $74.0 \pm 1.2$ \\
    SAGPool\textsubscript{h} \cite{lee2019self} & GCN & $3$ & $67.5 \pm 1.1$ & $67.9 \pm 1.4$  \\
    SAGPool\textsubscript{g} \cite{lee2019self} & GCN & $3$ & $74.2 \pm 1.2$ & $74.1 \pm 0.8$ \\
    \middashrule
    GIN \cite{errica2019fair, xu_how_2019} & GIN & $8$ & $80.0 \pm 1.4$ & - \\
    \middashrule
    GCN + VN \cite{ying_hierarchical_2018,gilmer2017neural} & GCN & $2$ & $71.5$ & -\\
    HGNet-EdgePool \cite{ying_hierarchical_2018, schlichtkrull2018modeling} & GCN+RGCN & $3+2$ & $77.1$ & - \\
    HGNet-Louvain \cite{ying_hierarchical_2018, schlichtkrull2018modeling} & GCN+RGCN & $3+2$ & $75.1$ & - \\
    \midrule
    GIN + FA \cite{alon2021on,xu_how_2019} & GIN & $8$ & $81.5 \pm 1.2$ & - \\
    \middashrule 
    GraphTrans (small) \cite{wu2021representing,vaswani2017attention} & GCN & $3$ & $81.3 \pm 1.9$ & $79.2 \pm 2.2$ \\
    GraphTrans (large) \cite{wu2021representing,vaswani2017attention} & GCN & $4$ & $\color{gray} \mathbf{82.6 \pm 1.2}$ & $\color{gray} \mathbf{82.3 \pm 2.6}$ \\
    \midrule
    \textbf{\GTAN} (ours) & \GTAN & $4$ & ${\mathbf{81.9 \pm 1.2}}$ & $78.0 \pm 2.3$ \\
    \textbf{\GTN} (ours) & \GTN & $4$ & $80.7 \pm 2.5$ & ${\mathbf{78.9 \pm 2.5}}$ \\
    \bottomrule
  \end{tabular}
  \end{small}
  \vspace{-5pt}
\end{table}

\begin{table}[ht]
\caption{Accuracies computed over 5 seeds of \GTAN ablated by changing a single option.}
\label{table:ablation_g2tan}
\vspace{-5pt}
\begin{adjustbox}{center}
  \centering
  \begin{small}
  \begin{tabular}{lccccccc}
    \toprule
    \textbf{Dataset} & \texttt{NoDiff} & \texttt{NoZeroStart} & \texttt{NoAlgOpt} & \texttt{NoJK} & \texttt{NoSkip} & \texttt{NoNorm} & \texttt{AvgPool} \\
    \midrule
    NCI1 & $80.1 \pm 0.7$ &  $79.5 \pm 1.8$ & $81.6 \pm 1.6$ & $82.1 \pm 1.8$ & $81.8 \pm 0.9$ & $81.6 \pm 1.5$ & $82.4 \pm 0.9$ \\
    NCI109 & $78.2 \pm 1.2$ & $77.7 \pm 1.8$ & $77.5 \pm 1.3$ & $77.6 \pm 1.2$ & $78.3 \pm 1.3$ & $79.8 \pm 1.4$ & $77.6 \pm 1.3$ \\
    \bottomrule
  \end{tabular}
\end{small}
\end{adjustbox}
\vspace{-10pt}
\end{table}

\paragraph{Results.} In Table \ref{table:nci_benchmark}, we report the mean and standard deviation of classification accuracy computed over 10 different seeds. Overall, both our models improve over all baselines in group \ref{baseline:local} on both datasets, maximally by $1.9\%$ on NCI1 and by $4.8\%$ on NCI109. In group \ref{baseline:pairwise}, \GTAN is solely outperformed by GraphTrans (large) on NCI1 by only $0.7\%$. Interestingly, the attention-free variation, \GTN, performs better on NCI109, where it performs very slightly worse than GraphTrans (small).

\paragraph{Ablation Study.} Further to using attention, we give a list of ideas for variations on our models in Appendix \ref{apx:variations}, which can be summarized briefly as: \begin{enumerate*}[label=(\roman*)] \item using increments of node attributes (\texttt{Diff}), \item preprending a zero point to sequences (\texttt{ZeroStart}), \item  optimizing over the algebra embedding (\texttt{AlgOpt}) \end{enumerate*}, all of which are built into our main models. Further, the previous architectural choices aimed at incorporating several commonly used tricks for training GNNs. We investigate the effect of the previous variations and ablate the architectural ``tricks'' by measuring the performance change resulting from ceteris paribus removing it. Table \ref{table:ablation_g2tan} shows the result for \GTAN. To summarize the main observations, the model is robust to all the architectural changes, removing the layer norm even improves on NCI109. Importantly, replacing the attention pooling with mean pooling does not significantly affect the performance, but actually slightly improves on NCI1. Regarding variations, $\texttt{AlgOpt}$ slightly improves on both datasets, while removing $\texttt{Diff}$ and $\texttt{ZeroStart}$ significantly degrades the accuracy on NCI1. The latter means that translations of node attributes are important, not just their distances. We give further discussions and ablations in Appendix \ref{apx:exp_detail} including \GTN.

\paragraph{Discussion.} The previous experiments demonstrate that our approach performs very favourably on long-range reasoning tasks compared to GNN-based alternatives without global pairwise node interactions. Several of the works we compare against have focused on extending GNNs to larger neighbourhoods by specifically designed graph coarsening and pooling operations, and we emphasize two important points: \begin{enumerate*}[label=(\arabic*)] \item our approach can efficiently capture large neighbourhoods without any need for coarsening, \item it already performs well with simple mean-pooling as justified by Theorem \ref{thm:characterizing_rw_informal} and experimentally supported by the ablations. \end{enumerate*}  Although the Transformer-based GraphTrans slightly outperforms our model potentially due to its ability to learn global interactions, it is not entirely clear how much of the global graph structure it is able to infer from interactions of short-range neighbourhood summaries. Intuitively, better and larger summaries should also help the Transformer submodule to make inference about questions regarding the global structure. Finally, Transformer models can be bottlenecked by their quadratic complexity in nodes, while our approach only scales with edges, and hence, it can be more favourable for large sparse graphs in terms of computations.

\section{Conclusion} 
\label{sec:conclusion}
Inspired by classical results from analysis \cite{gaveau_principe_1977}, we introduce the hypo-elliptic graph Laplacian. 
This yields a diffusion equation and also generalizes its classical probabilistic interpretation via random walks but now takes their history into account. 
In addition to several attractive theoretical guarantees, we provide scalable algorithms. 
Our experiments show that this can lead to largely improved baselines for long-range reasoning tasks. 
A promising future research theme is to develop improvements for the classical Laplacian in this hypo-elliptic context; this includes lazy random walks~\cite{xhonneux_continuous_2020}; nonlinear diffusions~\cite{chamberlain_grand_2021}; and source/sink terms~\cite{thorpe_grand_2022}.
Another theme could be to extend the geometric study \cite{topping22} of over-squashing to this hypo-elliptic point of view which is naturally tied to sub-Riemannian geometry \cite{strichartz1986sub}.

\section*{Acknowledgements}
C.T. is supported by a Mathematical Institute Award from the University of Oxford. C.H. and D.L. are supported by NCCR-Synapsy Phase-3 SNSF grant number 51NF40-185897.
H.O.~is supported by the DataSig Program [EP/S026347/1], the Alan Turing Institute, the Oxford-Man Institute, and the CIMDA collaboration by City University Hong Kong and the University of Oxford.

% \section*{References}

% References follow the acknowledgments. Use unnumbered first-level heading for
% the references. Any choice of citation style is acceptable as long as you are
% consistent. It is permissible to reduce the font size to \verb+small+ (9 point)
% when listing the references.
% Note that the Reference section does not count towards the page limit.
% \medskip

\bibliographystyle{plain}
\bibliography{higher_order_kernels}

\clearpage
\appendix

\section*{Appendix Outline}
Section \ref{apx:notation} summarizes the notation and objects used in this paper.
Section \ref{apx:tensors} gives general background on tensor and the algebra $\tensalg$.
Section \ref{apx:sequence features} discusses the feature map $\fms$ for sequences and possible choices for the lift $\fm$ from labels to the algebra $\tensalg$.
Section \ref{apx:details_diffusion} contains the proofs of our main theorems and some variations of  hypoelliptic diffusion.
Section \ref{apx:characterizing_rw} provides the formal statement of \Cref{thm:characterizing_rw_informal} and the extension to pooled features.
Section \ref{apx:low_rank_functionals} gives background and details of the low-rank algorithm.
Section \ref{apx:variations} discusses variations of the sequence feature map which lead to different node and graph features.
Section \ref{apx:exp} includes further experiments and discussion on the empirical results.

\section{Notation}
\label{apx:notation}

\renewcommand{\arraystretch}{1.25}
\begin{center}
\begin{longtable}{  m{0.13\textwidth}  m{0.81\textwidth}   } 
  \Xhline{1pt}
  Symbol & Meaning  \\ 
  \Xhline{1pt}
  \multicolumn{2}{c}{Fixed Parameters and Indices}\\
  \hline
  $d$ & dimension of node attributes \\
  $k$ & length of random walk \\
  $n$ & number of nodes in graph \\
    \hline
    \multicolumn{2}{c}{Sequence Features}\\
  \hline
  $\vs$ & finite dimensional vector space for node attributes \\
  $\Seq(\vs)$ & sequences $\bx=(x_0,\ldots,x_k)$ of arbitrary length $k$ in $\vs$ \\
  $\delta_k\bx$ & increments of a sequence where $\delta_0\bx \coloneqq x_0$ and $\delta_k\bx \coloneqq x_k - x_{k-1}$ for $k >0$ \\
  $\tensalg$  & tensor algebra of $\R^d$ (see~\Cref{apx:tensors}) \\
  $\fm$  & algebra lifting $\fm: \vs \rightarrow H$ \\
  $\exp_\otimes$ & tensor exponential $\exp_\otimes: \vs \rightarrow H$ (main example of algebra lifting) \\
$\fms$  & sequence feature map $\fms: \Seq(\vs) \rightarrow H$, where $\fms(\bx) = \fm(\delta_0 \bx) \cdots \fm(\delta_k \bx)$ \\
 \hline
    \multicolumn{2}{c}{Graphs, Adjacency and Laplacian Matrices}\\
  \hline
    $\cG$ & $\cG = (\cV, \cE, \nf)$ graph with vertex set, edge set, and node attributes $\nf: \cV \rightarrow \vs$\\
    $(B_k)_{k \ge 0}$ & simple random walk on graph (valued in $\cV$) \\
    $(L_k)_{k \ge 0}$ & lifted random walk over $\R^d$, $L_k \coloneqq f(B_k)$ \\
    $ \bL_k$ & lifted length $k$ random walk $\bL_k = (L_0, \ldots, L_k)$\\
$\fmn$ & $\fmn(i)=\E[\fms(L_1,\ldots,L_n)|L_0=i] \in H$ feature map for node $i \in \cV$\\
$\fmg$ & $\fmg(\cG) = \E[\fms(L_1,\ldots,L_n)] \in H$ feature map for graphs \\
$A$ & standard adjacency matrix \\
$\wA$  & tensor adjacency matrix \\
$P$ & standard transition matrix \\
$\wP$ & tensor transition matrix \\
$\cL$  & normalized graph Laplacian \\
$\wcL$ & normalized hypo-elliptic graph Laplacian \\
  \Xhline{1pt}
\end{longtable}
\end{center}
\renewcommand{\arraystretch}{1}
% \todo{check $L0$ included}
\paragraph{Notation Conventions}
\begin{itemize}
    \item Bold symbols are used for tensors and sequences; vectors are unbolded, such as $x \in \vs$
    \item Coordinates for vectors are denoted using superscripts: $x = (x^{(1)}, \ldots, x^{(d)}) \in \R^d$.
    % \item Following the coordinate convention, the basis vectors for $\R^d$ will be denoted as $e^1, \ldots, e^d$; and for a multi-index $I = (i_1, \ldots, i_m)$, the basis vector for $(\R^d)^{\otimes m}$ is denoted $e^I \coloneqq e^{i_1} \otimes \ldots \otimes e^{i_m}$
    \item Tensors are denoted by $\bv = (\bv_0, \bv_1, \bv_2, \ldots ) \in \tensalgps$, where $\bv_m \in (\vs)^{\otimes m}$.
    \item Sequences are denoted by $\bx = (x_0, \ldots, x_k) \in \Seq(\vs)$.
\end{itemize}
\section{Tensors and the Algebra \texorpdfstring{$\tensalg$}{H}}\label{apx:tensors}

% \begin{itemize}
%     \item background on tensors (sadly we need this!); could just recycle seq2tens appendix
%     \item example of how to add tensor series, inner product, linear functional
% \end{itemize}

In this section, we provide a brief overview of tensors on a finite-dimensional vector space $\R^d$, along with the resulting algebra $H$. 

\paragraph{Tensors on $\R^d$.}
While tensor products between vector spaces are defined more generally, our main focus is on defining the tensor powers of $\vs$, denoted by $(\vs)^{\otimes m}$ for some $m \in \N$. The tensor power $(\vs)^{\otimes m}$ is also a vector space. Given a basis $e^1, \ldots, e^d$ of $\vs$, we can define a basis of $(\vs)^{\otimes m}$ as the collection of all
\[
    e^I \coloneqq e^{i_1} \otimes \ldots \otimes e^{i_m}
\]
over all \emph{multi-indices} $(i_1, \ldots, i_m)$, where each $i_j \in [d]$. Thus, any element $\bv_m \in (\vs)^{\otimes m}$ can be represented as
\[
    \bv_m = \sum_{i_1, \ldots, i_m = 1}^d \bv_m^{(i_1, \ldots, i_m)} e^{(i_1, \ldots, i_m)},
\]
where the $\bv_m^{(i_1, \ldots, i_m)} \in \R$ specifies the coordinates of $\bv_m$. In particular, $(\vs)^{\otimes 1}$ is simply $\vs$ itself; $(\vs)^{\otimes 2}$ can be viewed as the space of $d \times d$ matrices; $(\vs)^{\otimes 3}$ is the space of $d \times d \times d$ arrays, etc.

Given two vectors $x, y \in \vs$ which we represent coordinate-wise as
\[
    x = (x^{(1)}, \ldots, x^{(d)}), \quad \quad y = (y^{(1)}, \ldots, y^{(d)}),
\]
the tensor product $x \otimes y \in (\vs)^{\otimes 2}$ is given by
\[
    (x\otimes y)^{(i,j)} \coloneqq x^{(i)} y^{(j)}.
\]
More generally, if $\bu_m \in (\vs)^{\otimes m}$ and $\bv_n \in (\vs)^{\otimes n}$, the tensor product $\bu_m \otimes \bv_n \in (\vs)^{\otimes (m+n)}$ is defined coordinate-wise by
\begin{equation}
\label{eq:tensor_product}
    (\bu_m \otimes \bv_n)^{(i_1, \ldots, i_{m+n})} \coloneqq \bu_m^{(i_1, \ldots, i_{m})} \bv_n^{(i_{m+1}, \ldots, i_{m+n})}.
\end{equation}
Furthermore, we note that $(\vs)^{\otimes m}$ inherits an inner product from $\vs$ through a choice of basis, as is shown in Equation~\eqref{eq:tensor_inner_product}.

\paragraph{Free Algebras.}
Our goal is to find a vector space $H$ that contains the vector space $\vs$ but is large enough to support a richer algebraic structure; in particular, a multiplication.
Formally, this means we look for an injective map $\vs \hookrightarrow H$ into an algebra $H$.
A classic mathematical construction that turns a vector space into an algebra is the \emph{free associative algebra}, defined as
\[
    \bigoplus_{m=0}^\infty (\vs)^{\otimes m} = \{ (\bv_0, \bv_1, \ldots, \bv_M, 0, \ldots) \, : \, \bv_m \in (\vs)^{\otimes m}, \, M \in \N\},
\]
where addition and multiplication of two elements $\bu = (\bu_0, \bu_1, \ldots)$ and $\bv = (\bv_0, \bv_1, \ldots)$ is given by defining the $(\vs)^{\otimes m} $ coordinate to be
\[
(\bu+ \bv)_m = \bu_m + \bv_m \in (\vs)^{\otimes m}, \quad \quad (\bu \cdot \bv)_m = \sum_{i=0}^m \bv_i \otimes \bw_{m-i} \in (\vs)^{\otimes m}.
\]
We emphasize that the multiplication is not commutative.
% An element $\bv=(\bv_m)_{m \ge 0}$ is a sequence of tensor $\bv_m \in (\vs)^{\otimes m}$ of increasing degree $m$; i.e.~$\bv_0$ is a scalar, $\bv_1$ is a vector, $\bv_2$ is a matrix, $\bv_3$ is an array with three indices, etc.
\paragraph{A Universal Property.}
Indeed, this construction is the most general way to turn $\vs$ into an algebra  as the following classical result shows: if we denote with $\iota: \vs \hookrightarrow \bigoplus_{m=0}^\infty (\vs)^{\otimes m}$ the embedding $\iota(x) = (0, x, 0, \ldots)$ then any linear map from $\vs$ into any associative algebra $A$ factors through $\iota$. 
In other words, given an associative algebra $A$ and any linear map $h: \vs \rightarrow A$, there exists a homomorphism of algebras $\tilde{h}: \bigoplus_{m=0}^\infty (\vs)^{\otimes m} \rightarrow H$ such the following diagram commutes
\[
\begin{tikzcd}[ampersand replacement=\&]
    \vs \ar{r}{h} \ar[swap]{d}{\iota} \& H \\
    \bigoplus_{m=0}^\infty (\vs)^{\otimes m}  \ar[dashed, swap]{ur}{\tilde{h}} \& .
\end{tikzcd}
\]
In a sense, this shows that $\bigoplus_{m=0}^\infty (\vs)^{\otimes m} $ is the ``most general algebra'' that $\vs$ embeds into. 
\paragraph{An Inner Product.}
While the free associative algebra satisfies the above universal property, it is convenient to work in the slightly larger space
\begin{equation}
    \prod_{m\geq 0} (\vs)^{\otimes m} = \{ (\bv_0, \bv_1,\bv_2,\ldots): \bv_m \in (\vs)^{\otimes m}\}
\end{equation}
and define an inner product
\begin{align}\label{eq: inner product sum}
    \langle \bu, \bv \rangle \coloneqq \sum_{m=0}^\infty \langle \bu_m, \bv_m\rangle_m
\end{align}
for elements  $\bu,\bv \in \prod_{m\geq 0} (\vs)^{\otimes m}$ for which the sum \eqref{eq: inner product sum} converges (it does not converge for general elements of $\prod_{m\geq 0} (\R^d)^{\otimes m}$).
The largest space for which this inner product exists, is 
\begin{equation}
    H \coloneqq \{ \bv \in \prod_{m\geq 0} (\vs)^{\otimes m} \, : \, \|\bv\| < \infty\} \text{ and }\|\bv\| \coloneqq \sqrt{ \langle \bv, \bv \rangle}.
\end{equation}
The space $H$ contains all the properties we require; it contains $\vs$ as a subspace, $\iota(\vs) \subset H$; the non-commutative multiplication structure, $\bu \cdot \bv \in H$ for $\bu,\bv \in H$; and a Hilbert space structure, $\langle \bu, \bv \rangle \in \R$ for $\bu,\bv \in H$. 

\paragraph{Linear Functionals.}
Any $\bell \in H$ can be treated as a linear functional by taking $\langle \bell, \cdot \rangle : H \rightarrow \R$, we will primarily consider \emph{finite linear functionals} 
\[
\bell = (\bell_0, \bell_1, \ldots, \bell_M, 0, \ldots) \in H,
\]
which have only finitely many nonzero coordinates.
Such finite linear functionals applied to our feature map $\fms(\bx) \in H$ for sequences $\bx \in \Seq(\vs)$, are rich enough so that
\[
\bx \mapsto \langle \bell, \fms(\bx) \rangle
\]
can approximate any functions $f(\bx)$ of sequences $\bx \in \Seq(\vs)$, and their collection characterizes the distribution of random sequences (i.e.~random walks),
\begin{align}
    \mu \mapsto \{ \langle \bell, \E_{\bx \sim \mu}[\fms(\bx)] \rangle: \bell=(\bell_0,\ldots,\bell_M,0,\ldots) \in H, M \ge 1 \}
\end{align}
is injective; see \Cref{thm:univ_char}.

\section{The Sequence Feature Map}\label{apx:sequence features}
In \Cref{sec:sequence features} we introduce the sequence feature map 
\begin{equation}
\fms: \Seq(\vs) \rightarrow H,\quad\fms(\bx) = \fm(\delta_0\bx) \cdots \fm(\delta_{k}\bx),
\end{equation}
for sequences in $\vs$ where
\[
\fm: \vs \to H
\]
is an injective map from $\vs$ into the algebra $H$. Here, $\delta_0 \bx\coloneqq x_0$ and $\delta_i\bx \coloneqq x_i - x_{i-1}$ for $i\ge 1$ denote the increments of a sequence $\bx=(x_0,\ldots,x_k)$.
Our main example is the tensor exponential
\begin{align}\label{eq:tensor exp}
\fm(x) = \exp_\otimes(x) = \left(\frac{x^{\otimes m}}{m!}\right)_{m \ge 0}.
\end{align}
The algebra $H$ is the free algebra discussed in detail in \Cref{apx:tensors}. 
In this case, a direct calculation shows that
\begin{align}\label{eq:discrete signature}
    \fms(\bx) = \left( \sum c(i_1,\ldots,i_m) \delta_{i_1} \bx \otimes \cdots \otimes \delta_{i_m} \bx \right)_{m \ge 0}
\end{align}
where the sum is taken over $i_1 \le \cdots \le i_m$ with $i_1,\ldots,i_m \in \{0,\ldots,k\}$ and the coefficients $c(i_1,\ldots,i_m) \in \R$ can be computed in explicit form, see \cite{toth_seq2tens_2020}.

\paragraph{Universality and Characteristicness}
Universality and characteristicness follow for our ``discrete time/sequence'' signatures $\fms(\bx) \in H$ from elementary arguments that we discuss below.
In our setting it is convenient to include the sequence coordinate in the lifting, that is we consider sequences 
\begin{equation}
\label{eq:time_parametrization}
\bar \bx = (\bar x_0, \ldots, \bar x_k), \text{ where } \bar x_i = (i, x_i) \in \R^{d+1}.
\end{equation}

\begin{theorem}
\label{thm:univ_char}
The \emph{time-parametrized sequence feature map} 
\begin{align}
\label{eq:fms_time_parametrized}
\fms: \Seq(\vs) \rightarrow H,\quad \fms(\bar \bx) = \fm(\delta_0 \bar \bx) \cdots \fm(\delta_k \bar \bx)
\end{align}
has the following properties:
\begin{enumerate}
    \item \textbf{Universality}: for any continuous\footnote{We use the metric $d(\bx, \by) = \sum_{i=1}^k \| (x_i - y_i) - (x_{i-1} - y_{i-1}) \|$ (if two sequences are of different lengths, we pad the end of the shorter one with the end point) to define the topology on $\Seq(\vs)$.}$f: \Seq(\vs) \to \R$, any $\epsilon >0$, and any compact set of sequences $K \subset \Seq(\vs)$, there exists a $\bell=(\bell_0,\bell_1,\ldots,\bell_m,0,0,\ldots)$ such that
    \begin{align}
        \sup_{\bx \in K} |f(\bar \bx) - \langle \ell, \fms(\bar \bx) \rangle| <\epsilon .
    \end{align}
    \item \textbf{Characteristicness}: let $\Prob(K)$ be the set of probability measures that are supported on a compact subset $K \subset \Seq(\vs)$. Then  the map
    \begin{align}
    \Prob(K) \to H, \quad \mu \mapsto \E_{\bx \sim \mu}[\fms(\bar \bx)]
    \end{align}
    is injective.
\end{enumerate}
\end{theorem}
\begin{proof}
    This is a folk theorem in control and probability theory: to see universality, it is sufficient to verify that $\{ \langle \bell, \fms(\bar \bx) \rangle:\, \bell=(\bell_0,\ldots,\bell_m,0,\ldots),m \ge 0\}$ is a point-separating algebra for $C(K,\R)$ since then the result follows from the Stone-Weierstrass theorem. 
    Point separation follows from \cite[Corollary 4.9]{fliess1981fonctionnelles}, and a direct calculation shows the product of $\langle \bell, \fms(\bar \bx) \rangle$ and $\langle \bell', \fms(\bar \bx) \rangle $ is again a linear functional of $\fms(\bar \bx)$. This shows that the collection of functionals forms an algebra.
    
    The characteristicness follows since $\Prob(K)$ is contained in the dual space of the space of continuous functions of sequences $C(K,\R)$ and by universality, linear functionals of $\fms(\bar \bx)$ are dense in this space, so the result follows.
\end{proof}
This result shows that:
\begin{itemize}
    \item non-linear functions of sequences can be approximated using linear functionals in $H$; and
    \item the distribution of random sequences is characterized as the mean of our feature map.
\end{itemize}
In the terminology of statistical learning this says that $\fms$ is a universal and characteristic feature map for the set $\Seq(V)$ of sequences.
For more background and extensions to non-compact sets of sequences or paths, we refer to \cite{kiraly_kernels_2019,chevyrev_signature_2018} and \cite[Section 3.2]{bonnier2020adapted}; for a more geometric picture see \cite{lee_path_2020}.

\paragraph{Path Signatures.} We now explain the remark made at the end of Section \ref{sec:sequence features}, that for the choice of $\fm$ as the tensor exponential~\eqref{eq:tensor_exponential}, the resulting sequence feature map $\fms$ can be identified as the time discretization of a classical object in analysis, called the path signature.
First, we lift a sequence $\bx=(x_0,\ldots,x_k)$ from discrete time to continuous time by identifying it as the piecewise linear path $\bX=(\bX(t))_{t \ge0 }$,
\[
\bX(t) \coloneqq x_{i} + (t-i)(x_{i+1}-x_i) \text{ for }t \in \left[{i},{i+1}\right).
\]
Then a direct calculation shows that
\begin{align}
\fms(\bx) = 
\left( \int_0^k d\bX^{\otimes m}\right)_{m \ge 0} \text{where}\int_0^i d\bX^{\otimes m}\coloneqq \int_{0 \le t_1 \le \cdots \leq t_m \le i}  \dot \bX(t_1) \otimes \cdots \otimes \dot \bX(t_m) d t_1 \cdots d t_m
\end{align}
and $i \in [0,k]$; note that $\dot \bX(t)$ is well-defined for almost every $t \in [0,k]$ since $\bX$ is piecewise linear, which is sufficient to make sense of this integral as a Riemann-Lebesgue integral (and stochastic integrals allow to treat rougher paths).
Such sequences of iterated integrals are classical in analysis, probability theory, and control theory, and are known under various names (Path-ordered Exponential, Volterra series, Chen--Fliess Series, Chronological Exponential, etc.); we refer to them as \emph{path signatures} as they are known in probability theory.
\paragraph{Diffusions and their Generator.}
The path signature can even can be defined for highly irregular paths such as Brownian motion by using stochastic (Ito-Stratonovich) integrals. This is the connection to our main motivation: if $\bX=(\bX_t)_{t \ge 0}$ is a Brownian motion in $\R^d$, then its generator is the classical Laplacian and the diffusion of Brownian particles is captured with the classical diffusion PDE, the heat equation. Gaveau \cite{gaveau_principe_1977} showed that if we lift Brownian trajectories $t \mapsto \bX(t)$ evolving in the state space $\R^d$ into paths evolving in a richer state space $H$ via the above signature construction,
\begin{align} \label{eq: lifted BM}
t \mapsto \left (\int_{0}^t d\bX^{\otimes m}\right)_{m \ge 0},
\end{align}
then the stochastic process~\eqref{eq: lifted BM} is again a Markov process\footnote{The process \eqref{eq: lifted BM} is often called the canonical lift of Brownian motion to the free Lie group with $d$ generators. Strictly speaking, Gaveau uses the first $M$ iterated integrals, otherwise one deals with an "infinite-dimensional Lie group" which poses some technical challenges. This Lie group embeds into $H$ and in this article we only use the structure of $H$ and not the Lie group structure; see \cite{gaveau_principe_1977}.} and its generator satisfies a property called hypo-ellipticity (see the paragraph below).
%Moreover, $\left(\int_{0}^k d\bX^{\otimes m}\right)_{m \ge 0}$ (essentially) characterizes the trajectory $(\bX(t))_{t \in [0,k]}$, see \cite{hambly_uniqueness_2010}, and linear functionals of $\left(\int_{0}^k d\bX^{\otimes m}\right)_{m \ge 0} \in H$  can approximate any continuous function of the path $(\bX_t)_{t \in [0,k]}$, e.g.~if $\sup_{s \in [0,t]} \|\bX(s)\|$ is a linear functional of $\left(\int_{0}^k d\bX^{\otimes m}\right)_{m \ge 0} \in H$ that depends on the whole trajectory of $\bX$ over the interval $[0,k]$ and not just the value at time $k$, $\bX(k)$.
Our approach can be seen as the analogous construction on a graph and in discrete time: in the same way that \eqref{eq: lifted BM} lifts Brownian motion from $\R^d$ to a process evolving in the state space $H$, the map $\fms$ lifts a simple random walk $(B_k)_{k \ge 0}$ on the graph to a random walk $\bL=(L_k)_{k \ge 0}$ in $H$.
% see also Figure \ref{fig:diffusions}.
\paragraph{Hypo-elliptic Operators, Sub-elliptic Operators, and their Geometry.}
A full discussion is beyond the scope of this article and several books \cite{hormander1967hypoelliptic} have been written about this topic. 
For the interested reader, we give a very short informal picture on the space $\R^d$ which motivates our nomenclature but otherwise this paragraph can be safely skipped.
A differential operator $\cL=\sum \sigma_{i,i} \frac{\partial^2}{\partial x_i \partial x_j} + \sum_i \mu_i \frac{\partial}{\partial x_j}$ is called elliptic if the matrix $(\sigma_{i,j}(x))_{i,j}$ is invertible for all $x \in \R^d$. 
Every (time-homogenous) Markov process $\bX=(\bX_t)_{t \ge 0}$ gives rise to a differential operator 
\[\cL f(x) = \lim_{t \to 0} \frac{\E[f(\bX_t)|\bX_0=x] - f(x)}{t}
\] 
but generically $\cL$ is not elliptic. 
An important example class of Markov processes are stochastic differential equations,
\begin{align}\label{eq:sde}
d\bX_t = \sum_{i=1}^e V_i(\bX_t) d \bB^i_t, \bX_0 \in \R^d,
\end{align}
where $\bB_t=(\bB_t^1,\ldots, \bB_t^e)$ denotes a Brownian motion in $\R^e$ and $V_1,\ldots,V_e$ are vector fields on $\R^d$.
By identifying the vectors fields $V_1,\ldots,V_e$ as differential operators, the generator of \eqref{eq:sde} can be written as
\begin{align}\label{eq:sum of squares}
\cL =\sum_{i=1}^e V_i^2,
\end{align}
which is in general not elliptic.
Sub- and hypo-ellipticity are properties of $\cL$ that are weaker than ellipticty: ellipticity implies sub-ellipticity and, by a classic theorem of H\"ormander, sub-ellipticity implies hypo-ellipticity.
Hypo-ellipticy is in turn general enough to cover many important examples for which (sub-)ellipticity fails.
For example, another celebrated result of H\"ormander is that the SDE generator \eqref{eq:sum of squares} is hypo-elliptic, whenever the Lie algebra generated by the vector fields and their brackets, spans at every point the full space $\R^d$. 
This not only allows us to study the properties of a large class of diffusions but also provides a natural link with geometry: the set of vector fields $V_1,\ldots,V_e$ determines a subset of $\R^d$ that the SDE evolves on. 
For general hypo-elliptic operators, this set is not all of $\R^d$ and it is not even a Riemannian manifold; rather, it is a sub-Riemannian manifold. 
Intuitively, sub-Riemannian geometry is "rougher" than Riemannian geometry (e.g.~no canonical connections exist) but still regular enough so that one can use geometric tools to study the underlying stochastic process and vice versa; see for example \cite{strichartz1986sub, gaveau_principe_1977}. 
This includes SDEs such as those evolving in free objects: for a natural choice of vector fields $V_1,\ldots,V_e$ the solution of the SDE \eqref{eq:sde} is \eqref{eq: lifted BM}, see \cite{gaveau_principe_1977}.

\section{Further Details on Hypo-elliptic Graph Diffusion} \label{apx:details_diffusion}

In Section \ref{sec:diffusion}, we introduced tensor analogues of classical matrix operators, which were then used to define hypo-elliptic graph diffusion. These tensor-valued matrices allow us to efficiently represent random walk histories on a graph, which can be manipulated using matrix operations. In this section, we discuss further details on the tensor adjacency matrix, provide a proof of~\Cref{thm:non_abelian_laplacian}, and introduce a variation of hypo-elliptic graph diffusion.

As in~\Cref{sec:diffusion}, we fix a labelled graph $\cG = (\cV, \cE, \nf)$, where the nodes are denoted by integers, $\cV = \{1, \ldots, n\}$, and $\nf: \cV \rightarrow \vs$ denotes the continuous node attributes. 

\paragraph{Powers of the Tensor Adjacency Matrix.} Recall that the powers of the classical adjacency matrix $A$ counts the number of walks between two nodes on the graph. In particular, given $k \in \N$, and two nodes $i,j \in \cV$, the result follows as a consequence of the sparsity pattern of $A$,
\begin{align}
\label{eq:classical_adj_powers}
    (A^k)_{i,j} = \sum_{i_1, \ldots, i_{k-1}=1}^n A_{i,i_1} \cdot A_{i_1, i_2} \cdots A_{i_{k-1},j} = \sum_{i=i_0 \sim \ldots \sim i_k = j} 1.
\end{align}
Note that the product $A_{i,i_1} \cdots A_{i_{k-1}, j} = 0$ unless each pair of consecutive indices are adjacent in the graph, namely $i_{q-1} \sim i_q$ for all $q = 1, \ldots, k$. Applying the same procedure to the tensor adjacency matrix from Equation~\eqref{eq:exponential_adjacency_matrix}, we obtain a summary of all walks between two nodes, rather than just the number of walks. In particular,

\begin{align}%star
     (\wA^k)_{i,j} & =\sum_{i=i_0\sim\cdots\sim i_k=j}\wA_{i,i_1}\cdot\wA_{i_1,i_2}\cdots \wA_{i_{k-1},j}  \\
    & = \sum_{i=i_0\sim\cdots\sim i_k=j}\fm(\nf(i_1)-\nf(i)) \cdot \fm(\nf(i_2)-\nf(i_1))\cdots \fm(\nf(j)-\nf(i_{k-1}))\\
    % &= \sum_{(i,i_1,\cdots,i_{k-1},j)}\fms(\nf(i),\nf(i_1),\ldots ,\nf(i_{k-1}), \nf(j)) \\
     &= \sum_{i=i_0\sim\cdots\sim i_k=j}\fm(\delta_1 \bx) \cdots \fm(\delta_k \bx),
\end{align}%star
where $\bx = (f(i_0), \ldots, f(i_k))$ denotes the lifted sequence in the vector space $\vs$. Note that this corresponds to the sequence feature map \emph{without} the initial point $\delta_0 \bx$.

\paragraph{Powers of the Tensor Transition Matrix.} We now consider powers of the classical transition matrix $P = I - \cL$, where $\cL = I - D^{-1} A$ is the normalized graph Laplacian. The entries of $P^k$ provide length $k$ random walk probabilities; in particular, we have
\begin{align}
    (P^k)_{i,j} = \sum_{i=i_0 \sim \ldots \sim i_k = j} \frac{1}{\deg(i_0) \deg(i_1) \ldots \deg(i_{k-1})} = \PP[B_k = j | B_0 = i].
\end{align}

The powers of the tensor transition matrix $\wP =I - \wcL$, where $\wcL = I - D^{-1}\wA$ is the hypo-elliptic Laplacian, will be the conditional expectation of the sequence feature map of the random walk process. In particular, 
\begin{align}\label{eq:powers_tensor_transition}
    (\wP)^k_{i,j} & = \sum_{i=i_0\sim\cdots\sim i_k=j} \fm(\delta_1 \bx) \cdots \fm(\delta_k \bx) \PP[B_1 = i_1, B_2 = i_2, \ldots, B_k = i_k | B_0 = i]. \\
    % &=\E[\fm(\delta_1 \bL_k) \cdots \fm(\delta_k \bL_k) | B_0 = i, B_k = j],
\end{align}
% where $\bL_k = (L_0, \ldots, L_k)$ is the graph random walk lifted to the vector space $L_i = f(B_i) \in \vs$. 

% In the same way as for the tensor adjacency we give an expression for the powers of the tensor transition matrix defined as $\wP =I-\wcL $.
% Written out explicitly, the coefficient of the $k$-th power of $I-\wcL$ for nodes $i,j\in\cV$ is the sum over all walks  $i =i_0\sim i_1\sim \cdots \sim i_k = j$ of length $k$ from node $i$ to node $j$:
% \begin{align}
%     (I-\wcL)^k_{i,j} & = \sum_{i=i_0\sim\cdots\sim i_k=j} \frac{\fm(\nf(i_1) - \nf(i_0))\otimes \cdots\otimes  \fm(\nf(i_k) - \nf(i_{k-1})) 
%     }{d_{i_0}\ldots d_{i_{k-1}}} \\
%     &=\sum_{i=i_0\sim\cdots\sim i_k=j} \fms(\nf(i_0),\cdots, \nf(i_k)) \PP(i_0\ldots i_k|B_0 = i),\label{eq:tensor_laplacian_entries}
% \end{align}
% This is the sum of features of all walks of length $k$ going from node $i$ to node $j$, weighted by the probability of each walk conditioned by the walk starting at node $i$.

\paragraph{Proof of Hypo-elliptic Diffusion Theorem.} Recall from~\Cref{sec:diffusion} that the \emph{hypo-elliptic graph diffusion equation} is defined by
\[
    \bv_k - \bv_{k-1} = -\wcL \bv_{k-1}, \quad \bv_0 = \one_H,
\]
where $\one^T_H = (1_H, \ldots, 1_H) \in H^n$ is the all-ones vector in $H$. Using the above computations for powers of the tensor transition matrix, we can prove~\Cref{thm:non_abelian_laplacian}, which is restated here.
\begin{customthm}{1}
    Let $k \in \N$, $\bL_k = (L_0, \ldots, L_k)$ be the lifted random walk from~\eqref{eq:lifted_random_walk}, and $\wP = I - \wcL$ be the \emph{tensor adjacency matrix}. The solution to the hypo-elliptic graph diffusion equation~\eqref{eq:non_abelian_diffusion_equation} is
    \[
        \bv_k = \left( \E[\fm(\delta_1\bL_k) \cdots \fm(\delta_k \bL_k) | B_0 = i]\right)_{i=1}^n = \wP^k \one_H.
    \]
    Furthermore, if $F \in H^{n \times n}$ is the diagonal matrix with $F_{i,i} = \fm(f(i))$, then
    \[
        F\bv_k = \left(\E[\fms(\bL_k)| B_0 = i]\right)_{i=1}^n.
    \]
\end{customthm}

% In the diffusion equation presented in \Cref{sec:diffusion}, we consider the case of \emph{forward diffusion}. 

% We provide a proof for \Cref{thm:non_abelian_laplacian}. 
\begin{proof}[Proof of~\Cref{thm:non_abelian_laplacian}]
We begin by proving the first equation. From the definition of hypo-elliptic diffusion, it is straightforward to see that $$\bv_k= (I-\wcL)\bv_{k-1} = \wP\bv_{k-1}  = \wP^k\bv_0 = \wP^k\one_H.$$

We prove coordinate-wise that $\bv_k = \left( \E[\fm(\delta_1\bL_k) \cdots \fm(\delta_k \bL_k) | B_0 = i]\right)_{i=1}^n $. Indeed, using the above and Equation~\eqref{eq:powers_tensor_transition}, the $i$-th coordinate of $\bv_k$ is 
\begin{align}%star
    \bv_k^{(i)} = (\wP^k\one_H)^{(i)} = \sum_{j=1}^n (\wP^k)_{i,j}  &= \sum_{i=i_0 \sim \ldots \sim i_k} \fm(\delta_1 \bx) \cdots \fm(\delta_k \bx) \PP[B_1 = i_1,  \ldots, B_k = i_k | B_0 = i] \\
    & = \E[\fm(\delta_1\bL_k) \cdots \fm(\delta_k \bL_k) | B_0 = i],
\end{align}%star
where $\bx = (f(i_0), \ldots, f(i_k))$ is the lifted sequence corresponding to a walk $i_0 \sim \ldots \sim i_k$ on the graph. 
Next, we will also prove the second equation coordinate-wise. Using the above result, we have
\begin{align}
    (F\bv_k)^{(i)} = \fm(f(i)) \E[\fm(\delta_1\bL_k) \cdots \fm(\delta_k \bL_k) | B_0 = i] = \E[\fms(\bL_k) | B_0 = i],
\end{align}
where we use the fact that $\delta_0\bL_k = L_0 = f(i)$ when we condition $B_0 = i$. 
\end{proof}

\paragraph{Forward Hypo-elliptic Diffusion.} 
In the classical setting, we can consider both the forward and backward Kolmogorov equations. Throughout the main text and in the appendix so far, we have been considering the \emph{backward} variants. In this section, we formulate the \emph{forward} analogue of hypo-elliptic diffusion.
In the classical graph setting, this corresponds to the following forward equation for $U_k \in \R^{n \times d}$ given by
\begin{align}
    U^T_k - U^T_{k-1} = - U_{k-1}^T\cL , \quad U^{(i)}_0 = \nf(i)
\end{align}
where the initial condition $U_0 \in \R^{n \times d}$ is specified by the node attributes. Note that because $P = D^{-1} A$ is right stochastic\footnote{Each column sums to one and hence multiplying on the left by a row vector conserves its mass.}, this variation of the graph diffusion equation conserves mass in each coordinate of the node attributes at every time step.

% $u_0\in \R^n $: 
% $$u_k - u_{k-1} = - u_{k-1}\cL , \quad u_0 = n^{-1}\one,$$
% where there is conservation of mass at each time step of the diffusion, contrarily to the diffusion equation \ref{eq:heat equation}. This is due to the fact the the transition matrix $P = D^{-1}A$ is left stochastic, i.e., each column sums to one and hence multiplying on the left by a vector conserves its mass.
 
Similarly we can formulate the forward hypo-elliptic graph diffusion equation for $\bv_k \in \tensalgps^n$ as
\begin{equation}
\label{eq:non_abelian_forward_equation}
    \bv^T_k - \bv^T_{k-1} = - \bv^T_{k-1}\wcL , \quad \bv_0 = n^{-1}\one_H.
\end{equation}
The solution of the Equation \ref{eq:non_abelian_forward_equation} is given below.  

\begin{theorem}
\label{thm:forward_diffusion}
    Let $k \in \N$, $\bL_k = (L_0, \ldots, L_k)$ be the lifted random walk from~\eqref{eq:lifted_random_walk}, and $\wP = I - \wcL$ be the \emph{tensor adjacency matrix}. The solution to the forward hypo-elliptic graph diffusion equation~\eqref{eq:non_abelian_forward_equation} is
    \[
        \bv^T_k = \left( \PP[B_k = i]\E[\fm(\delta_1\bL_k) \cdots \fm(\delta_k \bL_k) | B_k = i]\right)_{i=1}^n = \frac{1}{n}\one^T_H\wP^k.
    \]
    Furthermore, if $F \in H^{n \times n}$ is the diagonal matrix with $F_{i,i} = \fm(f(i))$, then
    \[
        \frac{1}{n}\one^T_H F \wP^k= \left(\PP[B_k = i] \E[\fms(\bL_k)| B_k = i]\right)_{i=1}^n.
    \]
\end{theorem}

\begin{proof}
The proof proceeds in the same way as the backward equation. By definition of the forward hypo-elliptic diffusion, we have
\[
    \bv_k^T = \bv_0^T \wP^k = \frac{1}{n} \one^T_H \wP^k.
\]

Now, recall that the initial point of the random walk process is chosen uniformly over all nodes; in other words, $\PP[B_0 = i] = \frac{1}{n}$. Then, we show $\bv^T_k = \left( \E[\fm(\delta_1\bL_k) \cdots \fm(\delta_k \bL_k) | B_k = i]\right)_{i=1}^n$ coordinate-wise as
\begin{align}
    \bv_k^{(i)} &= \frac{1}{n} \sum_{j=1}^n (\wP^k)_{j,i}\\
    &= \sum_{i_0 \sim \ldots \sim i_k = i} \fm(\delta_1 \bx) \cdots \fm(\delta_k \bx) \PP[B_1 = i_1,  \ldots, B_k = i | B_0 = i_0] \PP[B_0 = i_0] \\
    & = \sum_{i_0 \sim \ldots \sim i_k = i} \fm(\delta_1 \bx) \cdots \fm(\delta_k \bx) \PP[B_0 = j,  \ldots, B_{k-1} = i_{k-1} | B_k = i] \PP[B_k = i]\\
    & =  \PP[B_k = i] \E[\fm(\delta_1\bL_k) \cdots \fm(\delta_k \bL_k) | B_k = i],
\end{align}
where $\bx = (f(i_0), \ldots, f(i_k))$ is the lifted sequence corresponding to a walk $i_0 \sim \ldots \sim i_k$ on the graph. We will now prove the second equation. Note that we have
\[
    (F\wP^k)_{i,j} = \sum_{i=i_0 \sim \ldots \sim i_k = j} \fms(\bx) \PP[B_1 = i_1, \ldots, B_k = i_k \given B_0 = i].
\]
Then, following the same reasoning as the first equation, we obtain the desired result.
\end{proof}

% Assume that we pick the starting point of these random walks with initial distribution $\bv_0$, then $\bv_k = \bv_0(I-\wcL)^k$. Using the Equation \ref{eq:tensor_laplacian_entries} above we get
% \begin{align}%star
%   \bv_k& = \sum_{i=0}^{n}\bv_0^{(i)}(I-\wcL)^k_{i,j} \\
%   & = \sum_{i=0}^n \sum_{(i_0,\ldots, i_k)} \bv_0^{(i)} \frac{\fm(\nf(i_1) - \nf(i_0))\otimes \cdots\otimes  \fm(\nf(i_k) - \nf(i_{k-1})) 
%     }{d_{i_0}\ldots d_{i_{k-1}}}\\
%     & = \sum_{i=1}^n\sum_{(i_0,\ldots, i_k)} \PP(B_0 = i) \frac{\fm(\nf(i_1) - \nf(i_0))\otimes \cdots\otimes  \fm(\nf(i_k) - \nf(i_{k-1})) 
%     }{d_{i_0}\cdots d_{i_{k-1}}} \\
%     & = \sum_{i=1}^n\sum_{(i_0,\ldots, i_k)} \fms(\nf(i_0),\ldots, \nf(i_k)) \PP(i_0\cdots i_k|B_0 = i)\PP(B_0 = i)\\
%     & =\E[\fms(B_1,\ldots,B_n)|B_k =j ]
% \end{align}%star
% \end{proof}

\paragraph{Weighted Graphs.} In the prior discussion, we considered simple random walks in which the walk chooses one of the nodes  adjacent to the current node uniformly at random. We can instead consider more general random walks on graphs, which can be described using a weighted adjacency matrix,
\[
    A_{i,j} = \left\{ \begin{array}{cl}
        c_{i,j} & : i \sim j \\
        0 &: \text{otherwise}. 
    \end{array}\right. 
\]
In this case, we define the diagonal weighted degree matrix to be $D_{i,i} = \sum_{i \sim j} A_{i,j}$. We now define a weighted random walk $(B_k)_{k \ge 0}$ on the vertices $\cV$ where the transition matrix is given by
\[
    P_{i,j} \coloneqq D^{-1}A =\PP(B_k = j \given B_{k-1}=i) = \left\{ \begin{array}{cl}\frac{c_{i,j}}{\sum_{i \sim j'} c_{i,j'}} &: i\sim j\\ 0 &: \text{otherwise}.\end{array}\right.
\]
With these weighted graphs, the powers of the adjacency and transition matrix can be interpreted in a similar manner as the standard case. Powers of the adjacency matrix provide total weights over walks, while the powers of the transition matrix provide the probability of a weighted walk going between two nodes after a specified number of steps. In particular, 
\begin{align}
    (A^k)_{i,j} &= \sum_{i \sim i_1 \sim \ldots \sim i_{k-1} \sim j} c_{i,i_1} \ldots c_{i_{k-1}, j}\\
    (P^k)_{i,j} &= \PP[B_k = j \given B_0 = i].
\end{align}
The tensor adjacency and tensor transition matrices are defined in the same manner as
\begin{align}
    \wA_{i,j} &\coloneqq \left\{\begin{array}{cl} c_{i,j}\fm(f(j) - f(i)) & : i \sim j \\ 0 & : \text{otherwise} \end{array}\right. \\
    \wP_{i,j} &\coloneqq D^{-1} \wA = \left\{\begin{array}{cl} P_{i,j}\fm(f(j) - f(i)) & : i \sim j \\ 0 & : \text{otherwise} \end{array}\right.
\end{align}
Note that powers of this weighted tensor transition matrix are exactly the same as the unweighted case from Equation~\eqref{eq:powers_tensor_transition}, and thus the weighted version of both~\Cref{thm:characterizing_rw_informal} and~\Cref{thm:forward_diffusion} immediately follow.

\section{Characterizing Random Walks}
\label{apx:characterizing_rw}

In this appendix, we will provide further details on how the features obtained via hypo-elliptic diffusion are able to characterize the underlying random walk processes. These results rely on the characteristic property of the time-parametrized sequence feature map from~\Cref{thm:univ_char}.

\paragraph{Computation with Time Parametrization.}
Recall that the time-parametrized sequence feature map from Equation~\eqref{eq:fms_time_parametrized} applies the algebra lifting to the time-parametrized sequence $\bar\bx \coloneqq (\bar x_0, \ldots, \bar x_k)$, where $\bar x_i = (i, x_i)$. Note that we have $\delta_i \bar \bx = (1, x_i - x_{i-1})$. Thus, the hypo-elliptic diffusion equations and the low-rank algorithm given in~\Cref{thm:algo} easily extends to this setting.

\paragraph{Characterizing Random Walks.}
Recall that hypo-elliptic diffusion yields a feature map for labelled graphs by mean-pooling the individual node features as
\[
    \fmg(\cG) = \E[\fms(\bL_k)],
\]
where $\bL_k = (L_0, \ldots, L_k)$ is the lifted random walk process in $\vs$. We will now prove~\Cref{thm:characterizing_rw_informal}, which we restate here with more details.

\begin{customthm}{2}
    Let
    \begin{align}
    \fms: \Seq(\vs) \rightarrow H,\quad \fms(\bar \bx) = \fm(\delta_0 \bar \bx) \cdots \fm(\delta_k \bar \bx),
    \end{align}
    where $\bar\bx$ appends the time parametrization to the sequence $\bx$ as in Equation~\eqref{eq:time_parametrization}. Furthermore, suppose we have the resulting graph feature map
    \[
        \fmg(\cG) = \E[\fms(\bL_k)].
    \]
    Let $\cG$ and $\cG'$ be two labelled graphs, and $\bL_k = (L_0, \ldots, L_k)$ and $\bL'_k = (L'_0, \ldots, L'_k)$ be the $k$-step lifted random walk as defined in Equation~\eqref{eq:lifted_random_walk}, given the random walk processes $B$ and $B'$ on $\cG$ and $\cG'$ respectively. Then, $\Psi(\cG) = \Psi(\cG')$ if and only if the distributions of $\bL_k$ and $\bL'_k$ are equal. 
\end{customthm}

\begin{proof}
    First, if the distributions of the two random walks $\bL_k$ and $\bL'_k$ are equal, then it is clear that $\E[\fms(\bL_k)] = \E[\fms(\bL_k')]$.
    
    Next, suppose $\E[\fms(\bL_k)] = \E[\fms( \bL_k')]$. Then, note that the random walk distributions are finitely supported (hence compactly supported) distributions, and thus by~\Cref{thm:univ_char}, they must be equal.
\end{proof}

This result shows that the hypo-elliptic diffusion completely characterizes random walk histories; thus providing a highly descriptive summary of labelled graphs. There is an analogous result for the individual node features in terms of conditional expectations.

\begin{theorem}
    Let
    \begin{align}
    \fms: \Seq(\vs) \rightarrow H,\quad \fms(\bar \bx) = \fm(\delta_0 \bar \bx) \cdots \fm(\delta_k \bar \bx),
    \end{align}
    where $\bar\bx$ appends the time parametrization to the sequence $\bx$ as in Equation~\eqref{eq:time_parametrization}. Furthermore, suppose we have the resulting node feature map
    \[
        \fmn(i) = \E[\fms(\bL_k) | B_0 = i].
    \]
    Let $\cG$ and $\cG'$ be two labelled graphs, and $\bL_k = (L_0, \ldots, L_k)$ and $\bL'_k = (L'_0, \ldots, L'_k)$ be the $k$-step lifted random walk as defined in Equation~\eqref{eq:lifted_random_walk}, given the random walk processes $B$ and $B'$ on $\cG$ and $\cG'$ respectively. Then for two nodes $i \in \cV$ and $i' \in \cV'$ on the two respective graphs, $\fmn(i) = \fmn(i')$ if and only if the conditional distributions of $\PP[\bL_k \given B_0=i]$ and $\PP[\bL'_k \given B'_0=i']$ are equal. 
\end{theorem}
\begin{proof}
    The proof is analogous to the proof of~\Cref{thm:characterizing_rw_informal} given above. 
\end{proof}

% There is an analogous result for the mean-pooled random walk features for graph as defined in Equation~\eqref{eq:mean_pooled_features}.

% \begin{theorem} 
%     Given the notation from~\Cref{thm:random_walk_law}, 
%   \begin{align}
%     \E[\fms(\bL_k) ] &= \E[\fms( \bL'_k)] 
%   \end{align}
%   if and only if the random walk distributions $\PP[\bL_k]$ and $\PP[\bL'_k]$ are equal.
% \end{theorem}
% \begin{proof}
%     The proof is analogous to the proof of~\Cref{thm:random_walk_law} given above. 
% \end{proof}

\section{Details on the Low Rank Algorithm.}
\label{apx:low_rank_functionals}

In this section, we will provide further details and proofs on the low-rank approximation method discussed in~\Cref{sec:algos}. We can use low-rank tensors to define corresponding low-rank functionals of the features obtained via hypo-elliptic diffusion. The following is the definition of \emph{CP-rank} of a tensor from~\cite{carroll_analysis_1970}.

\begin{definition}
    The \emph{rank} of a level $m$ tensor $\bv_m \in (\vs)^{\otimes m}$ is the smallest number $r \geq 0$ such that we can express $\bv_m$ as
    \[
        \bv_m = \sum_{i=1}^r v_{i,1} \otimes \ldots \otimes v_{i,m}, \quad v_{i,j} \in \vs.
    \]
    We say that $\bv = (\bv_0, \bv_1, \ldots) \in H$ is \emph{rank $1$} if all $\bv_m$ are rank $1$ tensors. 
\end{definition}

We will now prove~\Cref{thm:algo}, which is stated using the node feature map without \texttt{ZeroStart} (see~\Cref{apx:variations}). In particular, the node feature map is given by
\begin{align}
\label{eq:fmn_no_zerostart}
    \fmnz_k(i) = \E[\fm(\delta_1 \bL_k) \cdots \fm(\delta_k\bL_k) \given B_0 = i] = (\wP^k \one_H)^{(i)} \in H,
\end{align}
where we explicitly specify the walk length $k$ in the subscript. Furthermore, if we also need to specify the tensor degree $m$, we will use two subscripts, where
\[
    \fmnz_{k,m}(i) \in (\vs)^{\otimes m}
\]
is the level $m$ component of the hypo-elliptic diffusion with a walk length of $k$. Throughout the proof, we omit the $H$ subscript for the all-ones vector, such that $\one^T \coloneqq (1_H, \ldots, 1_H) \in H^n$, and we denote $\one^T_i = (0, \ldots, 1_H, \ldots, 0) \in H^n$ to be the unit vector in the $i^{\text{th}}$ coordinate. 

\begin{proof}[Proof of~\Cref{thm:algo}]
% Throughout this proof, we let $\fmnz_{k,m} \in \left((\vs)^{\otimes m}\right)^n$ be the level $m$ component of the hypo-elliptic diffusion after $k$ steps across all $n$ nodes. 

First, we will show that $f_{1,m}(i) = \langle \bell_m, \fmnz_1(i)\rangle$ for all $m = 1, \ldots, M$. By the definition of hypo-elliptic diffusion, we know that
\[
    \fmnz_1(i) = \one_i^T \wP \one = \sum_{j=1} \wP_{i,j} = \sum_{i \sim j} \frac{\exp_{\otimes}(\nf(j) - \nf(i))}{d_i}.
\]
By explicitly expressing the level $m$ component, and by factoring out the inner product, we get
\begin{align}%star
    \left\langle \bell_m, \frac{\exp_{\otimes}(\nf(j) - \nf(i))}{d_i} \right\rangle &= \frac{1}{d_i M!} \prod_{r=m}^M \langle u_m, (\nf(j) - \nf(i))\rangle \\
    & = \frac{1}{M!} (P_{i,j} \cdot C^{u_{M-m}}_{i,j} \cdot \ldots \cdot C^{u_M}_{i,j}).
\end{align}%star
Then by linearity of the inner product, we get $f_{1,m}(i) = \langle \bell_m, \fmnz_1(i)\rangle$.

Next, we continue by induction and suppose that $f_{k-1,m}(i) = \langle \bell_m, \fmnz_{k-1}(i)\rangle$ holds for all $m = 1, \ldots, M$. Starting once again from the definition of hypo-elliptic diffusion, we know that 
\[
    \fmnz_k(i) = \one_i^T\wP\fmnz_{k-1} = \sum_{i \sim j} \wP_{i,j} \cdot \fmnz_{k-1}(j).
\]
Fix a degree $m$. We explicitly write out the level $m$ component of this equation by expanding $\wP_{i,j}$ and the tensor product as
\begin{align}
    \fmnz_{k,m}(i) &= \sum_{i \sim j} \sum_{r=0}^m \frac{(\nf(j) - \nf(i))^{\otimes r}}{d_i r!} \cdot  \fmnz_{k-1, m- r}(j) \nonumber\\
    & = \sum_{i \sim j} \frac{\fmnz_{k-1, m}(j)}{d_i} + \sum_{r=1}^m \frac{1}{r!} \sum_{i \sim j} \frac{(\nf(j) - \nf(i))^{\otimes r}}{d_i} \cdot  \fmnz_{k-1, m- r}(j) \label{eq:Phi_k_m_expanded}
\end{align}
Note that the first sum is equivalent to
\[
    \sum_{i \sim j} \frac{\fmnz_{k-1, m}(j)}{d_i} = \one_i^T P \cdot \fmnz_{k-1,m}.
\]

Applying the linear functional $\bell_m$ and the induction hypothesis to this, we have
\[
    \left\langle \bell_m, \sum_{i \sim j} \frac{\fmnz_{k-1, m}(j)}{d_i} \right\rangle = \one_i^T P \cdot f_{k-1,m}.
\]
For the second sum in Equation~\eqref{eq:Phi_k_m_expanded}, we can factor the inner product and apply the induction hypothesis to get
\begin{align}%star
    \Bigg\langle \bell_m , \sum_{i \sim j} \frac{(\nf(j) - \nf(i))^{\otimes r}}{d_i} \cdot  \fmnz_{k-1, m- r}(j) \Bigg\rangle  = \sum_{j=1}^n P_{i,j} \cdot C^{u_{M-m+1}}_{i,j} \cdot \ldots \cdot C^{M-m+r}_{i,j} \cdot f_{k-1, m-r}.
\end{align}%star
Putting this all together, we get
\begin{align}%star
    \fmnz_{k,m}(i) = \one_i^T \left( P \cdot f_{k-1, m} + \sum_{r=1}^m \frac{1}{r!} (P \odot C^{u_{M-m+1}} \odot \ldots \odot C^{u_M}) \cdot f_{k-1, m-r}\right) = f_{k,m}(i).
\end{align}%star
\end{proof}

\paragraph{Computing the ZeroStart Variation.}
We can adapt the recursive algorithm provided by~\Cref{thm:algo} above in order to compute low rank approximations to the ZeroStart variation of the node features, which is used throughout the main text (see also~\Cref{apx:variations}). In particular, we consider
\begin{align}
\label{eq:fmn_with_zerostart}
    \fmn_k(i) = \E[\fm(\delta_0\bL_k) \cdots \fm(\delta_k \bL_k) \given B_0 = i] = \fm(\delta_0\bL_k) \cdot \fmnz_k(i),
\end{align}
where $\fmnz_k$ is the variation without ZeroStart defined in Equation~\eqref{eq:fmn_no_zerostart}. Note that we can factor the $\fm(\delta_0\bL_k)$ term out of the expectation due to the conditioning $B_0 = i$, and thus $\fm(\delta_0\bL_k) = \fm(f(i))$ is fixed. Thus, we can compute low rank functionals of $\fmn_k(i)$ using one additional step.

\begin{theorem}
\label{thm:algo_zerostart}
    Using the same hypotheses as~\Cref{thm:algo}; let
    \[
        \bell_m = u_{M-m+1} \otimes \ldots \otimes u_M,
    \]
    where $u_m \in \vs$ for $m = 1, \ldots, M$ and let
    \[
        F^u_i \coloneqq \langle u, f(i)\rangle.
    \]
    Then,
    \begin{align}
        \langle \bell_M, \Phi_k(i)\rangle = \sum_{r=0}^M \frac{1}{m!} F^{u_1}_i \cdots F^{u_r}_i \cdot f_{k,M-r}(i),
    \end{align}
    where $f_{k,m}(i) = \langle \bell_m, \fmnz_k(i)\rangle$ from~\Cref{thm:algo}.
\end{theorem}
\begin{proof}
    Using the definition of $\Phi_k(i)$ from Equation~\eqref{eq:fmn_with_zerostart}, and expanding out the definition of $\fm(\delta_0\bL_k)$ at level $M$, we have
    \begin{align}
        \fmn_{k,M}(i) &= \sum_{r=0}^M \frac{f(i)^{\otimes r}}{r!} \cdot \fmnz_{k,M-r}(i).
    \end{align}
    Then, taking the linear functional and distributing, we obtain the result
    \[
        \langle \bell_M, \Phi_k(i)\rangle = \sum_{r=0}^M \frac{1}{r!} F^{u_1}_i \cdots F^{u_r}_i \cdot f_{k,M-r}(i).
    \]
\end{proof}

\paragraph{Computational Complexity.}
We will now consider the computational complexity of our algorithms. We begin by noting that the naive approach of computing $\Phi_k(i) = (F\wP^k\one_H)^{(i)}$ has the computational complexity of matrix multiplication; though this counts tensor operations, which itself requires $O(d^m)$ scalar multiplications at tensor degree $m$. This is computationally too expensive for practical applications.

Next, we consider the complexity of the recursive low-rank algorithm from~\Cref{thm:algo}, where the primary computational advantage is the fact that we only perform \emph{scalar} operations rather than \emph{tensor} operations. We consider the recursive step from Equation~\eqref{eq:recursive_low_rank}, reproduced here for $m = M$,
\[
   f_{k,M} \coloneqq P \cdot f_{k-1,M} + \sum_{r=1}^M \frac{1}{r!} \left(P \odot C^{u_{1}} \odot \cdots \odot C^{u_{r}}\right) \cdot f_{k-1,M-r}.
\]
Given a graph $\cG = (\cV, \cE, f)$ with $n = |\cV|$ nodes and $E = |\cE|$ edges, both $P$ and $C^u$ are sparse $n \times n$ (scalar) matrices with $O(E)$ nonzero entries, and $f_{k,m}$ is an $n$-dimensional column vector. Recall that $\odot$ denotes element-wise multiplication, and thus both the sparse matrix-matrix multiplication and the sparse matrix-vector multiplication have complexity $O(E)$. Furthermore, the entry-wise products $C^{u_1} \odot \cdots \odot C^{u_r}$ differ by only one factor between $r=m$ and $r=m+1$, and thus, computing $f_{k,M}$ assuming all lower $f_{k', m'}$ have been computed has complexity $O(M E)$. Taking into account the two recursive parameters results in a complexity of $O(kM^2 E)$. Note that this is the complexity to compute features for \emph{all nodes}.
% \todo{D: @Csaba, thanks for checking! I think this results in one fewer factor of $M$, which should be fixed now.}

Once the $f_{k,m}$ are computed, the complexity of adding the start point from~\Cref{thm:algo_zerostart} is $O(M)$.

\section{Variations and Hyperparameters of Hypo-Elliptic Diffusion}
\label{apx:variations}

In this appendix, we summarize possible variations of the sequence feature map, leading to different hypo-elliptic diffusion features. The choice of variation is learned during training, and we also summarize the hyperparameters used for our features. While the theoretical results on characterizing random walks, such as~\Cref{thm:characterizing_rw_informal}, depend on specific choices of the sequence feature map, there exist analogous results for these variations, which can characterize random walks up to certain equivalences. Furthermore, the computation of these variations can be performed in the same way: through tensorized linear algebra for exact solutions, and through an analogous low-rank method (as in~\Cref{thm:algo}) for approximate solutions. 

We fix an algebra lifting $\fm: \vs \rightarrow H$ and let $\bx = (x_0, \ldots, x_k) \in \Seq(\vs)$. The simplest sequence feature map to define simply multiplies the terms in the sequence together as
\[ \label{eq:nodiff}
    \fms(\bx) = \fm(x_0) \cdots \fm(x_k).
\]
Note that this is \emph{not} the sequence feature map used in the main text. We will now discuss several variations of this map, where $\fms_{\inc, \zs}$ is the one primarily used in the main text and $\fms_{\inc, \zs, \tp}$ is used to characterize random walks in~\Cref{thm:characterizing_rw_informal} and~\Cref{apx:characterizing_rw}.

\begin{description}
    \item [Increments ($\texttt{Diff}$).] Rather than directly multiplying terms in the sequence together, we can instead multiply the \emph{increments} as
    \[
    \fms_\inc(\bx) = \fm(\delta_1 \bx) \cdots \fm(\delta_k\bx),
    \]
    where $\delta_i \bx \coloneqq x_i - x_{i-1}$ for $i \geq 1$. In both cases, the sequence feature map is the path signature of a continuous piecewise-linear path when we set $\fm = \exp_\otimes$, as discussed in~\Cref{apx:sequence features}, and it is instructive to use this perspective to understand the effect of increments. If we use increments, the path corresponding to the sequence is
    \[
    \bX_\inc(t) \coloneqq x_{i} + (t-i)(x_{i+1}-x_i) \text{ for }t \in \left[{i},{i+1}\right),
    \]
    while if we do not use increments, the path corresponding to the sequence is
    \[
    \bX(t) \coloneqq \sum_{j=0}^{i - 1} x_j + (t-i)x_i \text{ for }t \in \left[{i},{i+1}\right).
    \]
    Thus, when we use increments the sequence $\bx$ corresponds to the vertices of the path $\bX_\inc$, while if we do not, it corresponds to the vectors between vertices of the path $\bX$. In practice, this variation corresponds to taking first-differences of the sequence $\bx$ before using eq.~\eqref{eq:nodiff}. 
    
    \item [Zero starting point ($\texttt{ZeroStart}$).] The sequence feature map with increments, $\fms_\inc$, as defined above is \emph{translation-invariant}, meaning $\fms_\inc(\bx+a) = \fms_\inc(\bx)$, where $\bx+a = (x_0 +a, x_1 +a, \ldots, x_k +a)$ for some $a \in \vs$. In order to remove translation invariance, we can start each sequence at the origin $0 \in \vs$ by pre-appending a $0$ to each sequence. A concise way to define the resulting \emph{zero started} sequence feature map is
    \[
    \fms_{\inc, \zs}(\bx) = \fm(\delta_0 \bx) \cdots \fm(\delta_k\bx),
    \]
    where we define $\delta_0 \bx \coloneqq x_0$. This is the sequence feature map defined in Equation~\eqref{eq:sequence_feature_map}. Note that this variation does not change the sequence feature map if we do not use increments.
    
    \item [Time parametrization.] When we relate sequences to piecewise linear paths as described in~\Cref{apx:sequence features}, we can use the fact that the path signature is invariant under reparametrization, or more generally, tree-like equivalence~\cite{hambly_uniqueness_2010}. In terms of discrete sequences, this includes invariance with respect to $0$ elements in the sequence (without increments), and repeated elements in the sequence (with increments). In order to remove this invariance, we can include \emph{time parametrization} by setting
    \[
        \fms_{-, \tp}(\bx) \coloneqq \fms_{-}(\bar{\bx}),
    \]
    where $\bar{\bx} \coloneqq (\bar x_0, \ldots, \bar x_k) \in \Seq(\R^{d+1})$, with $\bar x_i \coloneqq (i, x_i) \in \R^{d+1}$. This is a simple form of positional encoding, but other encodings are also possible, e.g.~sinusoidal waves as in \cite{vaswani2017attention}.
    
    \item [Algebra lifting ($\texttt{AlgOpt}$).] Throughout this article, we have used the tensor exponential as the algebra lifting. However, we can also scale each level of the lifting independently, and keep these as hyperparameters to optimize. In particular, for a sequence $\bc = (c_0, c_1, \ldots) \in \R^\N$, we define $\fm^\bc: \vs \rightarrow H$ to be
    \[
        \fm^\bc(x) \coloneqq \left(c_m x^{\otimes m}\right)_{m=0}^\infty,
    \]
    where $1 / m!$ is used as initialization for $c_m$ and learned along with the other parameters.
\end{description}

The choice of which variant of the sequence feature map to use depends on which invariance properties are important for the specific problem. In practice, the choice can be learned during the training, which is done in our experiments in~\Cref{sec:experiments}. Furthermore, the features obtained through the low-rank hypo-elliptic diffusion depend on three hyperparameters:
\begin{itemize}
    \item the length of random walks;
    \item the number of low-rank functionals;
    \item the maximal tensor degree;
    \item the number of iterations (layers).
\end{itemize}
Note that the first three hyperparameters can also potentially vary across different iterations. 

\section{Experiments} \label{apx:exp}

We have implemented the low-rank algorithm for our layers given in Theorem \ref{thm:algo} using Tensorflow, Keras and Spektral \cite{grattarola2021graph}. Code is available at \url{https://github.com/tgcsaba/graph2tens}. All experiments were run on one of 3 computing clusters that were overall equipped with 5 NVIDIA Geforce 2080Ti, 2 Quadro GP100, 2 A100 GPUs. The largest GPU memory allocation any one of the experiments required was around ${\sim}5$GB. For the experiments ran using different random seeds, the seed was used to control the randomness in the \begin{enumerate*}[label=(\arabic*)] \item data splitting process, \item parameter initialization, \item optimization procedure. \end{enumerate*} For an experiment with overall $n_{\text{runs}}$ number of runs, the used seeds were $\{0, \dots, n_{\text{runs}}-1\}$.
\subsection{Model details.} \label{apx:model_detail}
\begin{figure}[t]
    \centering
    \includegraphics[width=\textwidth]{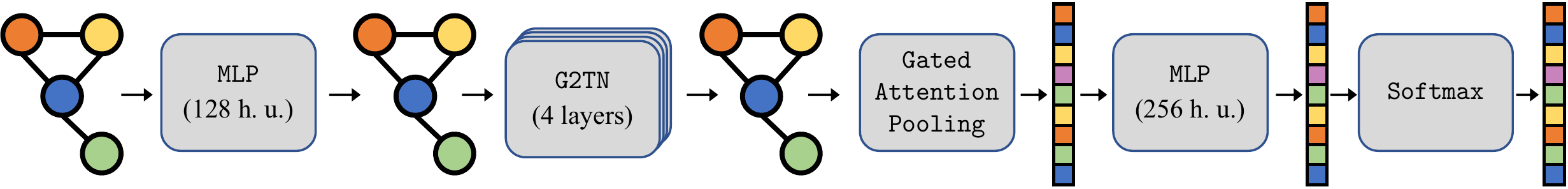}
    \caption{Visualization of the architecture used for NCI1 and NCI109 described in Section \ref{sec:experiments}.}
    \label{fig:architecture}
\end{figure}

The architecture described in Section \ref{sec:experiments} is visualized conceptually on Figure \ref{fig:architecture}, where for \GTAN the \GTN stack is replaced by the attentional variation without changing any hyperparameters. Further to the discussion given in the main text, here we provide more details on the model implementation.

\paragraph{Initialization.} In our \GTN and \GTAN layers, we linearly project tensor-valued features using linear functionals that use the same low-rank parametrization as the \emph{recursive} variation in \cite[App.~D.2]{toth_seq2tens_2020}. There, the authors also propose in Appendix D.5 an initialization heuristic for setting the variances of these component vectors, that we have employed with a centered uniform distribution.

\paragraph{Regularization.} We also apply $\ell_2$ regularization for both experiments in Sections \ref{sec:experiments} and \ref{apx:further_exp} for which we discuss the implementation here. Given a max.~tensor degree $M \geq 1$ and a max.~tensor rank $R \geq 1$ representing the layer width, there are overall $RM$ rank-$1$ linear functionals of the form
\[ \label{eq:lin_func}
    \bell^r_m = u^r_{M-m+1} \otimes \cdots \otimes u^r_{M} \quad \text{for } m=1,\dots M \text{ and } r=1, \dots, R.
\]
Since in practice these are represented using only the component vectors $u^r_{i} \in \R^d$, a naive application of $\ell_2$ regularization would lead to computing the $2$-norm of each component vector $u^r_i$ to penalize the loss function. However, we found this approach to underperform compared to the following idea. Although our algorithm represents the functionals using a rank-$1$ decomposition and computes the projection of each tensor-valued node feature without explicitly building high-dimensional tensors, conceptually we still have tensors ($\bell_m^r$) acting on tensors ($\fmn(i)$ for $i \in \cV$), and hence the tensor norm, $\norm{\bell_m^r}_2$, should be used for regularization. Fortunately, this can also be computed efficiently:
\[
    \norm{\bell_m^r}_2 = \norm{u_{M-m+1}^r \otimes \cdots \otimes u_M^r}_2 = \norm{u_{M-m+1}^r}_2 \cdots \norm{u_M^r}_2.
\]
Further, as is common, we replace the $\ell_2$ norm with its squared value, sum over all functionals in \eqref{eq:lin_func}, and multiply by the regularization parameter $\lambda > 0$ so that the final penalty is given by
\[
    \texttt{L2Penalty} = \lambda \sum_{r=1}^R \sum_{m=1}^M \norm{u_{M-m+1}^r}_2^2 \cdots \norm{u_M^r}_2^2.
\]

\subsection{Experiment Details.} \label{apx:exp_detail}
Here we provide further details on the main experiment described in Section \ref{sec:experiments}.
\paragraph{Hyperparameter Selection.}
 The model architecture was primarily motivated by GraphTrans (small) from \cite{wu2021representing}. Specifically, the number and width of GNN layers, and the dropout rate was adopted as is. The $\ell_2$ regularization strength was chosen equal to the weight decay rate. The other hyperparameters were tuned on a single split of NCI1. For the random walk length, we experimented with values $2, 5, 10$. For the given GNN depth ($4$), $2$ RW steps per layer was not enough to learn long-range interactions, while $10$ significantly slowed down the convergence rate during training. For the maximal degree of tensors, we experimented with values from $2,3,4$, and using values beyond $2$ did not provide improvements. Intuitively, the tensor degree represents the order of nonlinear interactions that are learnable by the layer, e.g.~a degree of $2$ encodes pairwise interactions between node features in the neighbourhood, while a higher degree of $M$ allows to encode interactions between certain $M$-tuples of nodes. We suggest that a degree of $2$ is a good baseline setting, and that increasing the GNN depth instead allows to efficiently capture higher order interactions, while increasing the effective influence radius at the same time. For the pre- and postprocessing layers, we experimented with various depths and choosing more than $1$ layer for each was counterproductive. The number of units was simply set to the GNN width ($128$) in the preprocessing layer, while for the postprocessing layer slightly increasing it was found to provide improvements ($256$), potentially to compensate for the large amount of information that is compressed in the pooling step.
 
 \begin{table}[t]
    \caption{Accuracies computed over 5 seeds of \GTN ablated by changing a single option.}
    \label{table:ablation_g2tn}
    \begin{adjustbox}{center}
      \centering
      \begin{small}
      \begin{tabular}{lccccccc}
        \toprule
        \textbf{Dataset} & \texttt{NoDiff} & \texttt{NoZeroStart} & \texttt{NoAlgOpt} & \texttt{NoJK} & \texttt{NoSkip} & \texttt{NoNorm} & \texttt{AvgPool} \\
        \midrule
        NCI1 & $79.1 \pm 1.5$ &  $80.7 \pm 0.6$ & $78.9 \pm 1.3$ & $79.4 \pm 1.9$ & $81.0 \pm 1.3$ & $79.5 \pm 1.1$ & $80.3 \pm 1.7$ \\
        NCI109 & $77.8 \pm 1.2$ & $79.5 \pm 1.5$ & $76.5 \pm 1.7$ & $78.1 \pm 1.9$ & $77.0 \pm 1.6$ & $77.4 \pm 2.7$ & $77.6 \pm 1.8$ \\
        \bottomrule
      \end{tabular}
    \end{small}
    \end{adjustbox}
\end{table}

 \paragraph{Further Ablation Studies.}
 Here we extend the investigation of our model architectures that was partially given in Section \ref{sec:experiments} in Table \ref{table:ablation_g2tan}. We give the analogous ablation study for the \GTN model in Table \ref{table:ablation_g2tn}, and compare the derived conclusions between the attentional and attention-free versions. First, we discuss the layer variations. Similarly to \GTAN, \texttt{NoDiff} slightly decreases the accuracy. However the conclusions regarding \texttt{NoZeroStart} and \texttt{NoAlgOpt} are different. In this case, removing \texttt{ZeroStart} actually improves the performance, while on \GTAN the opposite was true. An interpretation of this phenomenon is that only the attention mapping that is used to learn the random walks required information about translations of the layer's features, and not the tensor-features themselves. Another difference is that \texttt{NoAlgOpt} degrades the accuracy more significantly for \GTN. A possible explanation is that since \GTAN layers are more flexible thanks to their use of attention, they rely less on being able to learn the algebraic lift, while as \GTN layers are more rigid in their random walk definition, and benefit more from the added flexibility of \texttt{AlgOpt}. Additionally, it seems that \GTN is more sensitive to the various architectural options, and removing any of them, i.e.~\texttt{NoJK}, \texttt{NoSkip}, \texttt{NoNorm} or \texttt{AvgPool}, degrades the accuracy by $1\%$ or more on at least one dataset. Intuitively, it seems that overall the \GTAN model is more robust to the various architectural ``tricks'', and more adaptable due to its ability to learn the the random walk.

\subsection{Further Experiments} \label{apx:further_exp}
\paragraph{Citation datasets.}
Additionally to transductive learning on the biological datasets, we have carried out inductive learning tasks on some of the common citation datasets, i.e.~Cora, Citeseer \cite{sen2008collective} and Pubmed \cite{namata2012query}. We follow \cite{shchur2018pitfalls} in carrying out the experiment, and use the largest connected component for each dataset with $20$ training examples per class, $30$ validation examples per class, and the rest of the data used for testing.  The hyperparameters of our models and optimization procedure were based on the settings of the GAT model in \cite[Table 4]{shchur2018pitfalls}, which we have slightly fine-tuned on Cora and used for the other datasets. In particular, a single layer of \GTN or \GTAN is used with $64$ functionals, max.~tensor degree $2$ and random walk length $5$. The dropout rate was tuned to $0.9$, while the attentional dropout was set to $0.3$ in \GTAN. Optimization is carried out with Adam \cite{kingma2015adam}, a fixed learning rate of $0.01$ and $\ell_2$ regularization strength $0.01$. Training is stopped once the validation loss does not improve for $50$ epochs, and restored to the best checkpoint. For both of our models, \texttt{NoDiff} is used that we found to improve on the results as opposed to using increments of node features. The dropout rate had to be tuned as high as $0.9$, which suggests very strong overfitting, hence the additional complexity of \texttt{AlgOpt} was also contrabeneficial, and \texttt{NoAlgOpt} was used. As such, each model employs a single hidden layer, which is followed by layer normalization that was found to perform slightly better than other normalizations, e.g.~graph-level normalization.
\begin{table}[t]
\caption{Accuracies of our models on the citation datasets computed over 100 seeds compared with the 4 consistently best performing baselines from \cite{shchur2018pitfalls}.}
\label{table:cit_results}
\centering
\begin{adjustbox}{center}
\begin{small}
    \begin{tabular}{lcccccc}
    \toprule
    \textbf{Dataset} & \textbf{GCN} & \textbf{GAT} & \textbf{MoNet} & \textbf{GraphSage (mean)} & \textbf{\GTN (ours)} & \textbf{\GTAN (ours)} \\
    \midrule
    Cora & $81.5 \pm 1.3$ & $81.8 \pm 1.3$ & $81.3 \pm 1.3$ & $79.2 \pm 7.7$ & $\mathbf{82.6 \pm 1.0}$ & $82.0 \pm 1.1$ \\
    Citeseer & $\mathbf{71.9 \pm 1.9}$ & $71.1 \pm 1.9$ & $71.2 \pm 2.0$ & $71.6 \pm 1.9$ & $69.4 \pm 1.0$ & $68.2 \pm 1.3$ \\
    Pubmed & $77.8 \pm 2.9$ & $78.7 \pm 2.3$ & $78.6 \pm 2.3$ & $77.4 \pm 2.2$ & $\mathbf{78.8 \pm 1.9}$ & $78.0 \pm 1.9$ \\
    \bottomrule
    \end{tabular}
    \end{small}
\end{adjustbox}
\end{table}

\begin{table}[t]
    \caption{Accuracies of our models on $k$-hop sanitized citation datasets computed over 100 seeds compared with the 5 consistently best performing models from \cite{rampavsek2021hierarchical}.}
    \label{table:results_cit_khop}
    \centering
    \begin{adjustbox}{center}
    \begin{small}
        \begin{tabular}{llccccccc}
        \toprule
        \textbf{k} & \textbf{Dataset} & \textbf{GCN} & \textbf{GAT} & \textbf{g-U-net} & \textbf{HGNet-EP} & \textbf{HGNet-L} & \textbf{\GTN (ours)} & \textbf{\GTAN (ours)} \\
        \midrule 
        \multirow{3}{*}{$1$} & Cora & $76.7$ & $78.5$ & $78.1$ & $77.2$ & $77.1$ & $\mathbf{81.8 \pm 1.3}$ & $80.9 \pm 1.4$ \\
        & Citeseer & $64.2$ & $66.4$ & $63.0$ & $64.3$ & $64.10$ & $\mathbf{68.1 \pm 1.3}$ & $66.6 \pm 1.4$ \\
        & Pubmed & $75.8$ & $75.9$ & $75.8$ & $77.0$ & $76.3$ & $\mathbf{78.7 \pm 1.9}$ & $77.5 \pm 2.0$ \\
        \middashrule
        \multirow{3}{*}{$2$} & Cora & $72.0$ & $73.4$ & $74.4$ & $74.0$ & $75.4$ & $\mathbf{79.4 \pm 2.8}$ & $78.3 \pm 2.9$ \\
        & Citeseer & $58.3$ & $59.4$ & $57.3$ & $57.8$ & $59.9$ & $\mathbf{63.6 \pm 1.8}$ & $62.2 \pm 1.9$ \\
        & Pubmed & $72.1$ & $73.1$ & $72.4$ & $72.9$ & $75.1$ & $\mathbf{77.1 \pm 1.7}$ & $76.3 \pm 1.6$ \\
        \bottomrule
        \end{tabular}
    \end{small}
    \end{adjustbox}
    \end{table}

The results of our models over 100 runs are reported in Table \ref{table:cit_results} compared with the 4 consistently best performing baselines on these datasets from \cite{shchur2018pitfalls}, i.e.~GCN \cite{kipf_semi-supervised_2017}, GAT \cite{velickovic_graph_2018}, MoNet \cite{monti2017geometric}, and GraphSage \cite{hamilton2017inductive} with a mean aggregator. Firstly on Cora, both \GTN and \GTAN outperform the baselines with a more significant improvement for \GTN. For CiteSeer, our models are left somewhat behind compared to the baselines in terms of accuracy. Finally, they are again competitive on Pubmed, where \GTN takes the top score with a very slight lead. Two consistent observations are: \begin{enumerate*}[label=(\arabic*)] \item \GTN and \GTAN have a lower variance than all baselines, \item \GTN consistently outperforms \GTAN. \end{enumerate*} The latter may be attributed to the observation that due to the severe overfitting on these datasets, the additional complexity of the attention mechanism in \GTAN is unhelpful for generalization.

\paragraph{$K$-hop Sanitized Splits.} Recent work \cite{rampavsek2021hierarchical} has demonstrated that it is possible to make the previously considered citation datasets more suitable for testing the ability of a model to learn long-range information by dropping node labels in a structured way. Concretely, they use a label resampling strategy to guarantee that if a node is selected for a data split, none of its $k$-th degree neighbours are included in any splits, i.e.~training, validation nor testing, allowing to reduce the effect of short-range ``label imprinting''. In practice, we select a maximal independent set from the graph with respect to the $k$-th power of the adjacency matrix with self-loops, and repeat the previous experiment with the same data splitting method, model choice and training settings as before. In this case, the experiment seed is also used to control the random maximal independent set that is selected.  

The results of our models trained on the citation datasets sanitized this way are available in Table \ref{table:results_cit_khop} computed over 100 seeds for $k=1,2$. As baseline results, we compare against the 5 best performing models from \cite{rampavsek2021hierarchical}, where various GNN depths were also considered for each model, and we use the \emph{best} reported result for each of the baselines. Overall, we can observe that all baseline models exhibit a very sharp drop in performance as $k$ is increased, while for \GTN and \GTAN, the performance decrease is not nearly as pronounced. For both $k=1$ and $k=2$, our models perform better than the baselines on all datasets. This is explained by the fact that due to using random walks length of $5$, the models can efficiently pick up on information outside of the sanitized neighbourhoods. As previously, \GTN performs better than \GTAN as the additional flexibility of the attention layer does not lead to improved generalization performance when severe overfitting is present. This experiment demonstrates that the proposed models efficiently pick up on long-range information within larger neighbourhoods, and it suggests that on inductive learning tasks they should be more robust to sparse labeling rates compared to common short-range GNN models. 
% \putbib
% \end{bibunit}

\end{document}